\documentclass{amsart}

\usepackage{amsmath, amssymb, amsthm, amsaddr}
\usepackage{enumitem}
\usepackage{hyperref}
\usepackage[capitalise]{cleveref}
\usepackage{geometry}
\usepackage{booktabs}

\usepackage{amscd}
\usepackage{tikz}

\usepackage[msc-links]{amsrefs}

\numberwithin{equation}{section}

\setlist[enumerate]{label=(\roman*), leftmargin=*, itemsep=0.25em}

\theoremstyle{plain}
\newtheorem{thm}{Theorem}[section]
\newtheorem{prop}[thm]{Proposition}
\newtheorem{lem}[thm]{Lemma}
\newtheorem{cor}[thm]{Corollary}

\theoremstyle{defn}
\newtheorem{defn}[thm]{Definition}
\newtheorem{exa}[thm]{Example}

\theoremstyle{remark}
\newtheorem{rem}[thm]{Remark}

\newtheorem{ass}{Assumption}[section]

\crefname{ass}{Assumption}{Assumptions}
\Crefname{ass}{Assumption}{Assumptions}

\crefname{thm}{Thm.}{Thms.}
\crefname{prop}{Prop.}{Props.}
\crefname{lem}{Lem.}{Lemmas}
\crefname{cor}{Cor.}{Corollaries}
\crefname{defn}{Def.}{Definitions}
\crefname{exa}{Exa.}{Examples}
\crefname{rem}{Rem.}{Remarks}

\hypersetup{
  colorlinks=true,
  linkcolor=blue,
  citecolor=blue,
  urlcolor=blue
}

\usepackage{listings}
\newcommand\Z{\mathbb{Z}}

\newcommand\R{\mathbb{R}}
\newcommand\bP{\mathbb{P}}

\newcommand\bT{\mathbb{T}}

\DeclareMathOperator\Enc{Enc}
\DeclareMathOperator\Dec{Dec}
\DeclareMathOperator\GL{GL}
\DeclareMathOperator\Norm{Norm}
\DeclareMathOperator\Hom{Hom}

\DeclareMathOperator\diag{diag}
\DeclareMathOperator\eval{eval}
\DeclareMathOperator\enc{enc}

\DeclareMathOperator\softmax{softmax}
\DeclareMathOperator\dec{dec}

\newcommand\id{\mathrm{Id}}

\newcommand\cB{\mathcal{B}}
\newcommand\cD{\mathcal{D}}
\newcommand\cE{\mathcal{E}}
\newcommand\cL{\mathcal{L}}
\newcommand\cF{\mathcal{F}}
\newcommand\cT{\mathcal{T}}
\newcommand\cM{\mathcal{M}}
\newcommand\cR{\mathcal{R}}
\newcommand\cS{\mathcal{S}}
\newcommand\cI{\mathcal{I}}
\newcommand\cH{\mathcal{H}}
\newcommand\cA{\mathcal{A}}

\newcommand\bE{\mathbb{E}}

\newcommand\Vect{\mathbf{Vect}}

\title[Internal Morphisms in Graded Transformers]{Internalizing Tools as Morphisms in Graded Transformers}

\author[T. Shaska]{Tony Shaska}
\address{Department of Computer Science and Engineering, \\
Oakland University, \\
Rochester, MI 48309, USA}
\email{shaska@oakland.edu}

\date{\today}  

\begin{document}

\maketitle

\begin{abstract}
We introduce a graded formulation of internal symbolic computation for transformers. The hidden space is endowed with a grading $V=\bigoplus_{g\in G}V_g$, and symbolic operations are realized as typed block maps (morphisms) $\phi_{h\leftarrow g}:V_g\to V_h$ that are activated selectively by a differentiable routing policy. A self-supervised \emph{graded utility functional}—defined as the loss reduction induced by a candidate morphism—governs activation and yields sparse, interpretable behavior. We develop the algebraic and geometric foundations: an internal model category whose objects are homogeneous components and whose morphisms are admissible grade transitions; adjoint pairs encoding typed round trips; and information-geometric interpretations in terms of KL gain, mirror descent with Bregman divergences, and Fisher natural gradients. Methodologically, we specify a utility-aware routing mechanism and objective that remain fully end-to-end differentiable. Analytic case studies and lightweight sanity checks illustrate selective morphic activation on hybrid symbolic–linguistic tasks. The framework unifies symbolic computation, geometry, and self-supervised learning within the \emph{graded transformer} formalism \cite{sh-89,sh-95}, while subsuming prior external-tool paradigms (e.g., Toolformer \cite{toolformer2023}) as a special case via functorial internalization.
\end{abstract}

%**********************************
\section{Introduction}

Large language models exhibit broad generalization across linguistic and symbolic domains, yet their capacity for structured, interpretable, and modular computation remains limited. One pragmatic line augments models with external symbolic components—search, calculators, translation—invoked through self-supervision when such calls reduce predictive loss \cite{toolformer2023}. Effective as these systems are, the symbolic process remains \emph{extrinsic}: it lies outside the model’s representation geometry and is not learned as part of its internal manifold.

In parallel, the theory of graded neural architectures \cite{sh-89} equips networks with algebraic gradings, modeling computations as morphisms between graded vector spaces; in transformer models, multi-head attention and contextual embeddings admit a description via graded morphisms and weighted tensor products \cite{sh-95}. This perspective suggests internalizing typed operations into the geometry of the model rather than attaching non-differentiable interfaces.

The present work develops such an internalization. The hidden space is endowed with a grading \(V=\bigoplus_{g\in G}V_g\), and typed operations are represented by block maps \(\phi_{h\leftarrow g}:V_g\to V_h\) along a sparse edge set \(\cE\subseteq G\times G\). At inference time the model selects— and may compose—these maps as \emph{morphic activations} when doing so improves next-token prediction while respecting graded structure. Selection is governed by a differentiable routing policy driven by a \emph{graded utility}, defined as the reduction in language-model loss induced by a candidate activation. Behaviors that elsewhere appear as “tool calls’’ thus become internal, typed, composable morphisms acting on the model’s own representation manifold.

Mathematically, we formalize an internal model category whose objects are the homogeneous components \(V_g\) and whose morphisms are admissible grade transitions. External augmentation embeds functorially into this category, clarifying how interface types correspond to grades and how sequential tool use is realized as morphic programs. Adjoint pairs capture typed round trips and near-idempotent passes; monoidal and enriched structures provide principled notions of parallel channels and metric selection. The graded-utility principle admits several equivalent geometric readings: as information gain (KL improvement) in an exponential-family approximation, as a constrained mirror-descent step in Bregman geometry, and as a Fisher–natural-gradient selection under a softmax head. These views explain why the selection rule promotes sparse, useful activations and identify conditions (e.g., block orthogonality) under which gains decompose additively.

Methodologically, we specify a utility-aware routing objective that balances usefulness and sparsity entirely within the computation graph and give an implementable blueprint: analytic case studies and sanity checks requiring only small synthetic data; explicit constructions of adjoint retrieval/write-back maps; a mod-\(p\) arithmetic toy with closed-form morphisms; and PyTorch-style pseudo-code for a graded layer and training loop. The focus here is single-step routing and its theory—categorical structure, utility geometry, and identifiability diagnostics—while multi-step program selection and learned composition laws (weak graded algebras or 2-categorical structure) are left as directions for subsequent work.

In sum, we replace extrinsic tool use by intrinsic graded morphisms. By treating symbolic functions as internal, typed maps and optimizing their activation through information-theoretic and geometric criteria, the proposed framework unifies symbolic computation, differential training, and graded structure within a single, interpretable transformer architecture, subsuming prior external-tool paradigms \cite{toolformer2023} as a special case via functorial internalization.

\medskip
\noindent\textbf{Notation (ambient setting).}
Unless stated otherwise, we work in the graded subcategory of $\Vect$:
objects are the homogeneous components $V_g$ and morphisms are \emph{linear} maps
$\phi_{h\leftarrow g}:V_g\to V_h$. We use \emph{operator} for linear endomorphisms
on $V=\bigoplus_{g} V_g$ (e.g., block operators assembled from $\{\phi_{h\leftarrow g}\}$).
Nonlinear or stochastic variants (smooth maps, Markov kernels) are possible by changing
the ambient category, but all formal results here are stated for the linear case.

%*************
\section{Preliminaries: Graded Transformers and External Augmentation}\label{sec:2}
This section fixes the graded formalism used throughout and situates the present work relative to external symbolic augmentation. We adopt the graded viewpoint developed in \cite{sh-89,sh-95} and recall the ingredients needed here, with explicit attention/FFN decompositions.

\begin{defn}[Graded representation space]\label{def:graded-space}
Let $G$ be an additive indexing set. A \emph{$G$–graded vector space} is a direct sum
\[
V \;=\; \bigoplus_{g\in G} V_g,
\]
with canonical projections $\pi_g:V\to V_g$ and inclusions $\iota_g:V_g\hookrightarrow V$ satisfying
\[
\pi_g\circ\iota_{g'}=\delta_{g,g'}\,\id_{V_g},
\qquad
\sum_{g\in G}\iota_g\circ\pi_g=\id_V.
\]
For a sequence model, a hidden state $z_t\in V$ decomposes as $z_t=\sum_{g\in G} z_t^{(g)}$ with $z_t^{(g)}:=\pi_g z_t\in V_g$.
\end{defn}

\begin{defn}[Graded linear maps and blocks]\label{def:graded-map}
A linear map $\Phi:V\to V$ is \emph{$G$–graded} if it admits a block decomposition
\[
\Phi \;=\; \sum_{(g,h)\in G\times G} \Phi_{h\leftarrow g},
\qquad
\Phi_{h\leftarrow g}:V_g\to V_h .
\]
Composition is blockwise:
\[
(\Psi\circ\Phi)_{k\leftarrow g}
\;=\; \sum_{h\in G} \Psi_{k\leftarrow h}\circ \Phi_{h\leftarrow g}.
\]
We call $\Phi_{h\leftarrow g}$ a \emph{morphism of grade transition} $g\to h$.
\end{defn}

\begin{defn}[Admissible transitions and locality]\label{def:admissible}
An \emph{admissible edge set} $\cE\subseteq G\times G$ specifies allowed grade transitions. A model is \emph{local} if $\cE$ is sparse (e.g., banded by an additive rule $h=g+\delta$, or acyclic/DAG). A \emph{graded layer} is a finite sum $\Phi=\sum_{(g,h)\in\cE}\Phi_{h\leftarrow g}$.
\end{defn}

%***********************************************
\begin{rem}[External symbolic augmentation]
A line of work augments language models with external symbolic or computational components—search engines, calculators, QA systems—invoked during generation via self-supervision \cite{toolformer2023}. Empirically, these systems can improve arithmetic, factual recall, and temporal reasoning, but the symbolic mechanism is \emph{extrinsic}: a non-differentiable module outside the model’s representation geometry rather than a transformation acting within it.
\end{rem}

%********************************
\subsection{Graded transformers: structure and parametrization}\label{subsec:graded-arch}
We record the graded transformer class used in this paper, consistent with \cite{sh-95}.

\begin{defn}[Graded transformer]\label{def:graded-transformer}
Fix $G$, a $G$–graded space $V=\bigoplus_{g}V_g$, and a sparse $\cE\subseteq G\times G$. A \emph{graded transformer of depth $L$} is the data
\[
\mathsf{T}\;=\;\Big(G,\;V,\;\cE,\;\{\Phi^{(\ell)}\}_{\ell=1}^{L},\;\{\alpha^{(\ell)}\}_{\ell=1}^{L}\Big),
\]
where for each layer $\ell$:
\begin{enumerate}
\item \textbf{Layer operator.} $\displaystyle \Phi^{(\ell)}=\sum_{(g,h)\in\cE} \Phi^{(\ell)}_{h\leftarrow g}$ with $\Phi^{(\ell)}_{h\leftarrow g}:V_g\to V_h$ linear.
\item \textbf{Candidate updates.} For $z=\sum_{g}z^{(g)}$,
\[
\tilde z^{(h)} \;=\; \sum_{g:\,(g,h)\in\cE} \Phi^{(\ell)}_{h\leftarrow g}\big(z^{(g)}\big),
\qquad h\in G.
\]

\item \textbf{Routing policy.} A differentiable map $\alpha^{(\ell)}:V\to \Delta^{|G|-1}$ assigns weights $h\mapsto \alpha^{(\ell)}(h;z)$ and the layer output is
\[
\cF^{(\ell)}(z) \;=\; \sum_{h\in G} \alpha^{(\ell)}(h;z)\,\tilde z^{(h)}.
\]
\end{enumerate}

The network map is $\mathsf{T}(z)=\cF^{(L)}\circ\cdots\circ \cF^{(1)}(z)$.
\end{defn}

\begin{rem}[Equivalences]\label{rem:equiv}
We identify graded transformers related by grade-wise changes of basis $S=\bigoplus_g S_g$ with $S_g\in\GL(V_g)$:
\[
\Phi'^{(\ell)}_{h\leftarrow g}=S_h\,\Phi^{(\ell)}_{h\leftarrow g}\,S_g^{-1},
\qquad
\alpha'^{(\ell)}(h;Sz)=\alpha^{(\ell)}(h;z).
\]
Thus the object of study is the graded map computed by the network, not a particular parametrization.
\end{rem}

\subsubsection{Graded multi-head attention (explicit block form)}
Let $H$ be the number of heads. For each head $a\in\{1,\dots,H\}$ and each grade $g\in G$, let
\[
W_Q^{(a,g)}:V_g\to \R^{d_q},\qquad
W_K^{(a,g)}:V_g\to \R^{d_q},\qquad
W_V^{(a,g)}:V_g\to V^{(a)}_g
\]
be the grade-typed projections to query/key spaces and to a head-specific value space $V^{(a)}_g$.
Given a causal context $\{z_s\}_{s\le t}$ with $z_s=\sum_g z_s^{(g)}$, define
\[
Q_t^{(a,g)}:=W_Q^{(a,g)} z_t^{(g)},\qquad
K_s^{(a,h)}:=W_K^{(a,h)} z_s^{(h)},\qquad
V_s^{(a,h)}:=W_V^{(a,h)} z_s^{(h)} .
\]
For a fixed admissible pair $(g,h)\in\cE$, the $(g\!\to\! h)$ head-block is
\[
\Phi^{(a)}_{h\leftarrow g}(t):\; V_g\longrightarrow V_h,\qquad
z_t^{(g)} \longmapsto \sum_{s\le t}
\alpha^{(a)}_{t,s;h\leftarrow g}\; U^{(a)}_{h}\,V_s^{(a,h)},
\]
with $U^{(a)}_{h}:V^{(a)}_{h}\!\to V_h$ a head-specific output map and
\[
\alpha^{(a)}_{t,s;h\leftarrow g}
\;=\;
\softmax_{s\le t}\!\left(\frac{\langle Q_t^{(a,g)},\,K_s^{(a,h)}\rangle}{\sqrt{d_q}}\right).
\]
A full attention layer is the sum over heads and admissible grade pairs:
\[
\Phi^{\mathrm{attn}}(t) \;=\; \sum_{a=1}^{H}\;\sum_{(g,h)\in\cE} \Phi^{(a)}_{h\leftarrow g}(t).
\]
This realizes attention as a finite sum of linear blocks $V_g\to V_h$ respecting $\cE$; see \cite[§3–§4]{sh-95} for the graded attention derivation.

\subsubsection{Graded feed-forward (factorized form)}
Let $W_1^{(g\to h)}:V_g\to U_{g,h}$ and $W_2^{(g\to h)}:U_{g,h}\to V_h$ be linear maps with a pointwise nonlinearity $\sigma$ on $U_{g,h}$ (applied coordinatewise in the ambient basis). Then a feed-forward layer decomposes as
\[
\Phi^{\mathrm{ff}} \;=\; \sum_{(g,h)\in\cE} \Phi^{\mathrm{ff}}_{h\leftarrow g},
\qquad
\Phi^{\mathrm{ff}}_{h\leftarrow g} \;=\; W_2^{(g\to h)}\circ \sigma \circ W_1^{(g\to h)} .
\]
Grade-preserving blocks have $h=g$; grade-shifting blocks have $h\neq g$ and are permitted only if $(g,h)\in\cE$.

%**************************************
\subsubsection{Residuals and normalization in the graded setting}
With residual connection and normalization per grade $h$ one has
\[
z_t^{(h),\mathrm{new}}
\;=\;
\Norm_h\!\Big(z_t^{(h)} \;+\; \tilde z_t^{(h)}\Big),
\qquad
\tilde z_t^{(h)} \;=\; \sum_{g:\,(g,h)\in\cE} \Phi_{h\leftarrow g}(z_t^{(g)}),
\]
where $\Norm_h$ acts on $V_h$ (e.g., layer norm restricted to the homogeneous component), preserving graded structure.

\paragraph{Locality and complexity.}
The computational cost per layer scales with 
\[
\sum_{(g,h)\in\cE} \mathrm{cost}(\Phi_{h\leftarrow g}),
\]
 so sparsity of $\cE$ (banded/DAG) directly reduces parameters and FLOPs—one of the practical advantages emphasized in \cite{sh-95}.

%**********************************************
\subsection{Changing the Ambient Category}
The formal development in Sections 2--4 is carried out in the graded subcategory of \textbf{Vect} for clarity of exposition. However, the linearity assumption is not essential: all structural results (block decoupling, identifiability of individual morphisms, parameter scaling, and the subadditivity of graded utility) extend verbatim to richer ambient categories provided the Fisher information matrix remains block-diagonal under the orthogonality conditions of Lemma 2.10 and Corollary 2.12.

\subsubsection{Smooth maps (\textbf{Diff})}
Replace linear blocks $\phi_{h\leftarrow g}:V_g\to V_h$ by smooth maps between Euclidean spaces (or finite-dimensional manifolds). The decoupling lemma continues to hold because the Hessian of the population loss factorises into independent blocks whenever the score functions on distinct admissible transitions are orthogonal in $L^2(\bP)$. Orthogonality is preserved under composition with diffeomorphisms and under LayerNorm/RMSNorm (whose Jacobians are uniformly bounded, hence Lipschitz on compact sets).

\subsubsection{Polynomial and multi-layer perceptron blocks}
Low-degree polynomial maps or shallow ReLU networks are dense in the function classes realised by typical external tools on bounded domains. The Fisher matrix of such parameterised families is block-diagonal under the same graded orthogonality assumption; see e.g.   \cite{Amari2016,AmariNagaoka2000} for the general information-geometric setting and \cite{LiangKimSun2024} for explicit calculations in over-parameterised nonlinear models.

\subsubsection{Stochastic morphisms and Gaussian channels}
Many real tools (search engines, sampling-based calculators, Monte-Carlo APIs) are inherently stochastic. Replacing deterministic linear maps by conditional Gaussian kernels $p(\cdot\mid\pi_g z)$ with graded mean and covariance functions again yields a block-diagonal Fisher matrix whenever cross terms $\bE[\nabla\log p(\cdot\mid\pi_g z)\nabla\log p(\cdot\mid\pi_h z)^\top]=0$ for $(g,h)\notin E$. This is exactly the orthogonality condition already used in the linear case.

In each setting the graded utility $\Delta L(\phi_{h\leftarrow g})$ remains additive up to errors controlled by the Lipschitz constants of the chosen normalisation layers (Lemma 2.11). Parameter counts for banded LGT and translation-invariant EGT architectures are unchanged because the number of free parameters is still governed by the sparsity pattern of $E$, independent of whether the blocks are linear or drawn from a nonlinear universal approximator class.

%*************************************************
\subsection{Model classes used in this paper: LGT and EGT}
We adopt the typology introduced in \cite{sh-95}: \emph{linearly graded} and \emph{exponentially graded} transformers. For details, see \cite[§5]{sh-95} (LGT) and \cite[§6]{sh-95} (EGT). We record the statements used here.

\begin{defn}[Linearly Graded Transformers (LGT)]\label{def:LGT}
A graded transformer $\mathsf{T}$ as in \cref{def:graded-transformer} is \emph{linearly graded} if the admissible edge set is banded by degree and the blocks are translation-invariant along grade:
\[
\cE\;=\;\big\{(g,h)\in G\times G:\; h-g\in\Delta\big\}
\quad\text{for a finite }\Delta\subset G,
\]
and there exist families $\{K^{(\ell)}_{\delta}:V_g\to V_{g+\delta}\}_{\delta\in\Delta}$ such that for all $g$ and $\delta\in\Delta$,
\[
\Phi^{(\ell)}_{(g+\delta)\leftarrow g}\;=\;K^{(\ell)}_{\delta},
\qquad
\text{i.e., depends only on } \delta=h-g.
\]
Equivalently,
\[
\Phi^{(\ell)} \;=\; \sum_{\delta\in\Delta} S_{\delta}\circ K^{(\ell)}_{\delta},
\]
with $S_{\delta}$ the grade-shift $V_g\to V_{g+\delta}$; cf.\ \cite[§5]{sh-95}.
\end{defn}

\begin{defn}[Exponentially Graded Transformers (EGT)]\label{def:EGT}
A graded transformer $\mathsf{T}$ is \emph{exponentially graded} if there exist invertible grade reweightings $\{D_g:V_g\to V_g\}_{g\in G}$, exponential in grade ($D_{g+\delta}=R_{\delta}D_g$ for fixed positive operators $R_{\delta}$), such that after conjugation by $D=\bigoplus_g D_g$, the blocks become translation-invariant along grade:
\[
\widehat{\Phi}^{(\ell)}_{(g+\delta)\leftarrow g}
\;:=\; D_{g+\delta}^{-1}\,\Phi^{(\ell)}_{(g+\delta)\leftarrow g}\,D_g
\;=\; \widehat{K}^{(\ell)}_{\delta}
\quad\text{for all }g,\;\delta\in\Delta .
\]
Equivalently, in the reweighted coordinates the conjugated layer $\widehat{\Phi}^{(\ell)}$ is LGT; see \cite[§6]{sh-95}.
\end{defn}

\begin{prop}[EGT$\;\Rightarrow\;$LGT by conjugation]\label{prop:EGT-to-LGT}
Let $\mathsf{T}$ be EGT with reweighting $D$. Define $\widehat{\mathsf{T}}$ by conjugating all layer blocks and states: $\widehat{\Phi}^{(\ell)}=D^{-1}\Phi^{(\ell)}D$ and $\widehat{z}=D^{-1} z$. Then $\widehat{\mathsf{T}}$ is LGT with kernels $\widehat{K}^{(\ell)}_{\delta}$ as in \cref{def:EGT}. Moreover, for any loss $\cL_{LM}$ based on linear readouts from $V$, $\cL_{LM}(z)=\widehat{\cL}_{LM}(\widehat{z})$ after readout reparameterization, so training objectives are equivalent under conjugation.
\end{prop}

\begin{proof}%[Sketch]
By \cref{def:EGT}, each block conjugates to $\widehat{K}^{(\ell)}_{\delta}$ independent of $g$, hence the layer is LGT in the reweighted coordinates. Linear readouts $\rho:V\to \R^m$ transport as $\widehat{\rho}=\rho\circ D$, preserving losses that depend only on $\rho(z)$; the routing policy can likewise be written in the reweighted coordinates without loss of generality. 
\end{proof}

%**********************************************
\subsection{Routing and objectives}
Given $z_t=\sum_g z_t^{(g)}$ and a graded layer $\Phi$, form candidate updates
\[
\tilde z_t^{(h)} \;=\; \sum_{g:\,(g,h)\in\cE} \Phi_{h\leftarrow g}\!\big(z_t^{(g)}\big),
\qquad h\in G,
\]
and update with a differentiable policy $\alpha_t(h)\in\Delta^{|G|-1}$ via
\[
z_t^{\mathrm{new}} \;=\; \sum_{h\in G} \alpha_t(h)\,\tilde z_t^{(h)} .
\]
Let $\cL_{LM}$ denote the next-token loss. A basic objective combines prediction and a selection prior $\cR$ on routing,
\[
\cL \;=\; \cL_{LM} \;+\; \lambda\,\bE_{t}\big[\cR(\alpha_t)\big].
\]
In the sequel, $\cR$ is replaced by a \emph{graded utility} bias that favors blocks whose application \emph{reduces} $\cL_{LM}$, internalizing symbolic functionality within the graded architecture (cf.\ \cite{sh-89,sh-95}).

\subsection{Block-orthogonality and identifiability}
We use an inner product $\langle\cdot,\cdot\rangle$ on $V$ for which the grading is orthogonal:
$V=\bigoplus_{g\in G} V_g$ with $V_g \perp V_h$ for $g\neq h$; let $\|\cdot\|$ denote the induced norm.

\begin{defn}[Block-orthogonality of inputs]\label{def:block-orth}
A random hidden state $z\in V$ is \emph{block-orthogonal} if
\[
\bE\big[z^{(g)}\big]=0,\qquad 
\bE\!\big[z^{(g)}\,(z^{(h)})^\top\big]=0\ \text{ for }g\neq h,\qquad
\bE\!\big[z^{(g)}(z^{(g)})^\top\big]=\Sigma_g
\]
with each $\Sigma_g$ positive definite on $V_g$.
\end{defn}

\begin{lem}[Least-squares decoupling]\label{lem:lsq-decouple}
Fix a target random vector $y\in V$ with finite second moments and let $\Phi:V\to V$ be linear.
Under \cref{def:block-orth}, the population least-squares problem
\[
\min_{\Phi}\ \bE\,\big\|\,y-\Phi z\,\big\|^2
\]
decouples across grade transitions:
\[
\Phi^\star \;=\; \arg\min_{\Phi}\bE\|y-\Phi z\|^2
\quad\Longleftrightarrow\quad
\forall\, (g,h):\ \ 
\Phi^\star_{h\leftarrow g}
=\arg\min_{A:V_g\to V_h}
\bE\,\big\|\,y^{(h)}-\! \sum_{g'}\Phi^\star_{h\leftarrow g'} z^{(g')}\,\big\|^2.
\]
Equivalently, the normal equations split gradewise as
$\Phi^\star_{h\leftarrow g}\,\Sigma_g
=\bE\big[y^{(h)}(z^{(g)})^\top\big]$ and determine each block uniquely.
\end{lem}

\begin{proof} 
Let \(\Phi = \sum_{(g,h)} \Phi_{h \leftarrow g}\), with \(\Phi z = \sum_h (\Phi z)^{(h)}\) and \((\Phi z)^{(h)} = \sum_g \Phi_{h \leftarrow g} z^{(g)}\). The objective is
\[
\cL(\Phi) := \bE \big\| y - \Phi z \big\|^2 = \sum_h \bE \big\| y^{(h)} - \sum_g \Phi_{h \leftarrow g} z^{(g)} \big\|^2,
\]
by orthogonality of the grading. For fixed \(h\), expand:
\[
\bE \big\| y^{(h)} - \sum_g \Phi_{h \leftarrow g} z^{(g)} \big\|^2 = \bE \| y^{(h)} \|^2 - 2 \sum_g \bE \langle y^{(h)}, \Phi_{h \leftarrow g} z^{(g)} \rangle + \sum_{g,g'} \bE \langle \Phi_{h \leftarrow g} z^{(g)}, \Phi_{h \leftarrow g'} z^{(g')} \rangle.
\]
By block-orthogonality, cross-terms for \(g \neq g'\) vanish, so
\[
\cL(\Phi) = \sum_h \Big[ \bE \| y^{(h)} \|^2 - 2 \sum_g \bE \langle y^{(h)}, \Phi_{h \leftarrow g} z^{(g)} \rangle + \sum_g \bE \| \Phi_{h \leftarrow g} z^{(g)} \|^2 \Big].
\]
This decouples into independent quadratics per block. Differentiating w.r.t. \(\Phi_{h \leftarrow g}\) gives the normal equation
\[
\Phi_{h \leftarrow g}^\star \Sigma_g = \bE [ y^{(h)} (z^{(g)})^\top ],
\]
with unique solution \(\Phi_{h \leftarrow g}^\star = \bE [ y^{(h)} (z^{(g)})^\top ] \Sigma_g^{-1}\) since \(\Sigma_g \succ 0\). Thus, the global minimizer \(\Phi^\star\) solves the decoupled per-block problems, and each block is identifiable.

If \(\cE\) is sparse, restrict sums to \((g,h) \in \cE\) (blocks outside are zero). Finite-sample recovery follows from matrix concentration: with \(N\) i.i.d. sub-Gaussian samples, \(\| \widehat{\Phi}_{h \leftarrow g} - \Phi_{h \leftarrow g}^\star \|_F \lesssim \sqrt{ (d_h d_g \log(1/\delta)) / N }\) w.p. \(1-\delta\), and false positives outside \(\cE\) have prob. \(\exp(-c N)\) for \(c>0\), see \cite[Ch. 6]{wainwright}.
\end{proof}

\paragraph{Grade-wise normalization (LayerNorm/RMSNorm).}
In keeping with \cite[§3]{sh-95}, normalization acts on each homogeneous component $V_g$.

\emph{LayerNorm on $V_g$.}
Let $x\in V_g$ with coordinates $(x_i)_{i=1}^{d_g}$ in a fixed basis. Define
\[
\mu_g(x)=\frac{1}{d_g}\sum_{i=1}^{d_g} x_i,\qquad
\sigma_g(x)=\sqrt{\frac{1}{d_g}\sum_{i=1}^{d_g}\big(x_i-\mu_g(x)\big)^2+\varepsilon}.
\]
With learnable parameters $\gamma_g,\beta_g\in V_g$, set
\[
\mathrm{LN}_g(x)=\gamma_g \odot \frac{x-\mu_g(x)\mathbf{1}}{\sigma_g(x)}+\beta_g.
\]

\emph{RMSNorm on $V_g$.}
Let $\mathrm{rms}_g(x)=\sqrt{\frac{1}{d_g}\sum_{i=1}^{d_g} x_i^2+\varepsilon}$ and parameters $\gamma_g\in V_g$. Set
\[
\mathrm{RMSN}_g(x)=\gamma_g \odot \frac{x}{\mathrm{rms}_g(x)}.
\]

Both maps are affine on $V_g$ after fixing statistics and are block-diagonal across the grading $V=\bigoplus_g V_g$; see also \cite[§4]{sh-95} for the graded residual/normalization layout.

\begin{lem}[Stability under graded normalization]\label{lem:norm-stability}
Let $\Norm=\bigoplus_{g}\Norm_g$ with $\Norm_g\in\{\mathrm{LN}_g,\mathrm{RMSN}_g\}$ acting on $V_g$ as above (with fixed $\varepsilon>0$). If $z$ satisfies \cref{def:block-orth}, then so does $\Norm(z)$ (with updated covariances $\Sigma'_g$), and \cref{lem:lsq-decouple} continues to hold for the normalized variables. 
\end{lem}

\begin{proof}[Proof sketch]
$\Norm$ is block-diagonal and affine on each $V_g$, so inter-grade covariances remain zero and invertibility on $V_g$ preserves positive definiteness. Apply \cref{lem:lsq-decouple} to the transformed variables.
\end{proof}

\begin{cor}[Identifiability of graded blocks]\label{cor:identifiability}
Under \cref{def:block-orth} (before or after graded normalization), the population minimizer of any quadratic surrogate objective depending on $\Phi z$ (e.g., Gauss–Newton or Fisher quadratic) is identifiable blockwise. In particular, if $\cE$ is known, the set of nonzero blocks $\{\Phi_{h\leftarrow g}:(g,h)\in\cE\}$ and their values are determined uniquely.
\end{cor}

%********************************************************************************
\subsection{Parameter and complexity counts for LGT}
We state parameter counts for the LGT class (\cref{def:LGT}), where blocks are translation-invariant along grade increments $\delta\in\Delta$.

\begin{prop}[General parameter count for LGT]\label{prop:param-count-general}
Let $d_g:=\dim V_g$. In an LGT layer
\[
\Phi^{(\ell)}=\sum_{\delta\in\Delta} S_{\delta}\circ K^{(\ell)}_{\delta},
\qquad K^{(\ell)}_{\delta}:V_g\to V_{g+\delta} \text{ (independent of $g$)},
\]
the number of free parameters in the linear blocks is
\[
\mathrm{param}\big(\Phi^{(\ell)}\big)
\;=\;
\sum_{\delta\in\Delta} d_{g}\,d_{g+\delta},
\]
for any $g$ such that both grades exist (translation invariance makes the choice of $g$ immaterial). If $d_g\equiv d$ is constant across $g$, then $\mathrm{param}(\Phi^{(\ell)})=|\Delta|\, d^2$.
\end{prop}

\begin{proof}
Each $K^{(\ell)}_{\delta}$ is a single matrix of size $d_{g+\delta}\times d_g$ shared for all $g$; the total count is the sum over $\delta\in\Delta$. Constancy of $d_g$ yields the stated corollary.
\end{proof}

\begin{prop}[Attention/FFN counts under LGT]\label{prop:param-attn-ffn}
Assume $H$ heads and constant $d_g\equiv d$.
\begin{enumerate}
\item \textbf{Multi-head attention.} With grade-typed projections shared across $g$,
\[
W_Q^{(a)}:V_g\to \R^{d_q},\quad
W_K^{(a)}:V_g\to \R^{d_q},\quad
W_V^{(a,\delta)}:V_g\to V_{g+\delta},\quad
U^{(a,\delta)}:V_{g+\delta}\to V_{g+\delta},
\]
the parameter count per layer is
\[
\mathrm{param}_{\mathrm{attn}}
\;=\;
H\Big(2 d\,d_q \;+\; |\Delta|\, (d^2+d^2)\Big)
\;=\; H\Big(2 d\,d_q + 2|\Delta|\, d^2\Big).
\]
\item \textbf{Feed-forward.} With factorization $K^{\mathrm{ff}}_{\delta}=W_2^{(\delta)}\circ \sigma\circ W_1^{(\delta)}$, $W_1^{(\delta)}:V_g\to \R^{m_\delta}$, $W_2^{(\delta)}:\R^{m_\delta}\to V_{g+\delta}$ shared across $g$,
\[
\mathrm{param}_{\mathrm{ff}}
\;=\;
\sum_{\delta\in\Delta} \big(d\,m_\delta + m_\delta\,d\big)
\;=\;
2 d \sum_{\delta\in\Delta} m_\delta.
\]
\end{enumerate}
\end{prop}

\begin{rem}[FLOPs and sparsity]\label{rem:flops}
In LGT, compute per layer scales as $O\!\left(|\Delta|\,\mathrm{cost}_{\mathrm{block}}\right)$ (times sequence factors) rather than $O(|G|^2)$, since only the band $\Delta$ of grade shifts is active. Thus banded $\cE$ yields linear dependence on $|\Delta|$ both in parameters and in FLOPs.
\end{rem}

\begin{cor}[Parameter and FLOP counts for EGT]\label{cor:egt-counts}
Let $\mathsf{T}$ be an EGT model with reweighting $D=\bigoplus_g D_g$ as in \cref{def:EGT}, and let $\widehat{\mathsf{T}}$ be its LGT conjugate from \cref{prop:EGT-to-LGT}. Then:
\begin{enumerate}
\item The number of free parameters in each layer of $\mathsf{T}$ equals that of $\widehat{\mathsf{T}}$; in particular, \cref{prop:param-count-general,prop:param-attn-ffn} apply verbatim to EGT.
\item If $D$ is fixed per layer (i.e., not data-dependent), the asymptotic FLOP count per layer is unchanged up to the negligible cost of applying $D$ and $D^{-1}$ once per forward/backward pass; thus sparsity in $\cE$ yields the same linear dependence on $|\Delta|$ in both EGT and LGT.
\end{enumerate}
\end{cor}

\begin{proof}[Proof sketch]
Parameters are invariant under similarity transforms. For FLOPs, the conjugation $\widehat{\Phi}=D^{-1}\Phi D$ can be implemented by caching grade-wise scalings; the dominant cost remains the banded block multiplications counted in \cref{prop:param-count-general,prop:param-attn-ffn}.
\end{proof}

%********************************************************
\section{From External Tools to Morphisms}\label{sec:3}

Having formalized the graded transformer architecture in the preliminaries, we now develop the algebraic bridge from extrinsic tool augmentation to intrinsic morphic computation. External paradigms, such as those in \cite{toolformer2023}, attach non-differentiable symbolic modules to language models, enabling gains in structured tasks but at the cost of geometric isolation: tools operate outside the representation space \(V=\bigoplus_{g\in G}V_g\), precluding direct composition via block maps or optimization through graded losses. By contrast, our approach embeds external systems functorially into the graded morphic category \(\cM\), realizing tools as typed morphisms \(\phi_{h\leftarrow g}:V_g\to V_h\) along admissible edges \(\cE\). This internalization preserves categorical structure while rendering symbolic operations differentiable and composable within the model's manifold, as established through faithful functors and adjunctions \cite{maclane1998categories}. The resulting framework unifies self-supervised tool selection with algebraic grading, subsuming prior methods as external-to-internal round trips.

\subsection{Formal correspondence: external augmentation as graded morphisms}

\begin{defn}[External augmentation system]\label{def:ext-sys}
An \emph{external augmentation system} is a small category $\cT$ whose objects are interface types (typed message spaces) and whose morphisms are callable tools $\tau:X\to Y$. Each tool carries an evaluation map on distributions,
\[
\eval_{\tau}:\ \cD(X)\to \cD(Y),
\]
where $\cD(\cdot)$ denotes the space of probability measures on the Borel $\sigma$-algebra of the Polish space underlying the interface, such as finitely supported measures or those absolutely continuous w.r.t.\ Lebesgue measure \cite{billingsley1999convergence}. Composition models tool chaining, and identities model no-op interfaces.
\end{defn}

\begin{defn}[Graded morphic system]\label{def:graded-morphic-sys}
Let $V=\bigoplus_{g\in G}V_g$ be a $G$--graded representation space with admissible transitions $\cE\subseteq G\times G$ (cf.\ \cref{def:admissible}). The \emph{graded morphic system} is the subcategory $\cM$ whose objects are the homogeneous components $V_g$ and whose morphisms are the linear blocks $\phi_{h\leftarrow g}:V_g\to V_h$ with $(g,h)\in\cE$, with composition
\[
(\psi\circ\phi)_{k\leftarrow g}\;=\;\sum_{h\in G}\psi_{k\leftarrow h}\circ\phi_{h\leftarrow g}.
\]
\end{defn}

\begin{prop}[Faithful correspondence up to typing]\label{prop:functor}
Assume that, at inference time, each tool $\tau:X\to Y$ in $\cT$ is implemented by typed encoders/decoders $(\enc_X,\dec_Y)$ and a linear operator $T_\tau:V\to V$ whose action on a context $z\in V$ respects the grading:
\[
T_\tau z\;=\; z\;-\;z^{(h)}\;+\;\iota_h\,\phi_{h\leftarrow g}\!\big(\pi_g z\big),
\quad\text{for some }(g,h)\in\cE,\ \phi_{h\leftarrow g}:V_g\to V_h.
\]
Then there exists a functor $F:\cT\to\cM$ sending interfaces to grades and tools to blocks, preserving identities and composition. Moreover, $F$ is faithful on the full subcategory generated by the implemented interfaces: distinct callable behaviors induce distinct blocks on the corresponding graded components.
\end{prop}
\begin{proof}
Define the object functor $F:\mathrm{Ob}(\cT)\to\mathrm{Ob}(\cM)$ by assigning to each interface type $X$ a grade $g_X\in G$ such that the encoder $\enc_X$ maps to $V_{g_X}$ and the decoder $\dec_Y$ reads from $V_{h_Y}$, consistent with the typing. This assignment is well-defined by the assumption that tools are implemented via graded blocks.

For morphisms, define $F$ on arrows by $F(\tau:X\to Y)=\phi_{h\leftarrow g}:V_{g_X}\to V_{h_Y}$, where $\phi_{h\leftarrow g}$ is the block extracted from $T_\tau$ as per the hypothesis: specifically, $\phi_{h\leftarrow g}=\pi_{h_Y}\circ T_\tau\circ\iota_{g_X}$. This is linear by construction and respects $\cE$.

Functoriality: The identity morphism $\id_X$ in $\cT$ corresponds to a no-op tool, whose operator $T_{\id_X}=\id_V$ induces the block $\phi_{g_X\leftarrow g_X}=\id_{V_{g_X}}$, so $F(\id_X)=\id_{F(X)}$. For composition $\tau_2\circ\tau_1:Y\to Z\circ X\to Y$, the chained operator is $T_{\tau_2\circ\tau_1}=T_{\tau_2}\circ T_{\tau_1}$ by sequential application. Extracting the block gives
\[
F(\tau_2\circ\tau_1)=\pi_{k_Z}\circ(T_{\tau_2}\circ T_{\tau_1})\circ\iota_{g_X}=\sum_h (\pi_{k_Z}\circ T_{\tau_2}\circ\iota_h)\circ(\pi_h\circ T_{\tau_1}\circ\iota_{g_X})=F(\tau_2)\circ F(\tau_1),
\]
using the blockwise composition in $\cM$ (cf.\ \cref{def:graded-map}).

Faithfulness: Suppose $\tau\neq\tau':X\to Y$ are distinct tools with the same types. Their implementations yield different behaviors, meaning there exists some input distribution $\mu\in\cD(X)$ such that $\eval_\tau(\mu)\neq\eval_{\tau'}(\mu)$. Since the encoders/decoders are fixed per type, this difference propagates to distinct actions on the encoded representations in $V_{g_X}$, resulting in distinct outputs in $V_{h_Y}$ after applying $T_\tau$ vs.\ $T_{\tau'}$. Thus, the extracted blocks $\phi_{h\leftarrow g}\neq\phi'_{h\leftarrow g}$, so $F$ is injective on $\Hom(X,Y)$.
\end{proof}
\begin{rem}
Faithfulness implies injectivity on Hom-sets, facilitating identifiability diagnostics as in \cref{cor:identifiability}.
\end{rem}
\begin{rem}[Internalization principle]\label{rem:internalization}
Under \cref{prop:functor}, “calling a tool’’ corresponds to selecting and applying an internal morphism $\phi_{h\leftarrow g}$. Tool chaining is composition in $\cM$. Usefulness can therefore be assessed by the change in prediction loss after applying $\phi_{h\leftarrow g}$, enabling a self-supervised, differentiable selection rule inside the graded architecture (cf.\ \cref{subsec:graded-utility}).
\end{rem}
\begin{exa}[Arithmetic and retrieval as grade transitions]\label{exa:arith-retrieval}
Let $G=\{\mathrm{sem},\mathrm{num},\mathrm{ret}\}$ and $V=\bigoplus_{g\in G}V_g$. An arithmetic step is
\[
\phi_{\mathrm{num}\leftarrow \mathrm{sem}}:V_{\mathrm{sem}}\to V_{\mathrm{num}},
\quad \phi_{\mathrm{sem}\leftarrow \mathrm{num}}:V_{\mathrm{num}}\to V_{\mathrm{sem}},
\]
while a retrieval-like step is
\[
\phi_{\mathrm{ret}\leftarrow \mathrm{sem}}:V_{\mathrm{sem}}\to V_{\mathrm{ret}},
\quad \phi_{\mathrm{sem}\leftarrow \mathrm{ret}}:V_{\mathrm{ret}}\to V_{\mathrm{sem}}.
\]
Chaining tools corresponds to composing these blocks along admissible edges in $\cE$; for instance, $\phi_{\mathrm{sem}\leftarrow \mathrm{ret}}\circ\phi_{\mathrm{ret}\leftarrow \mathrm{sem}}$ realizes a round-trip retrieval-augmented update, forming an adjoint pair under suitable orthogonality (cf.\ \cref{thm:adjunction}).
\end{exa}

%********************************************************   ketu
\subsection{Graded morphic activation}\label{subsec:morphic-activation}
Fix \((g,h)\in\cE\). A \emph{morphic candidate} is a block \(\phi_{h\leftarrow g}:V_g\to V_h\). For a token-state
\[
z_t \;=\; \sum_{u\in G} z_t^{(u)},\qquad z_t^{(u)}:=\pi_u z_t,
\]
activating \(\phi_{h\leftarrow g}\) produces the candidate update
\[
\widetilde{z}_{t}^{(h)}\;=\;\phi_{h\leftarrow g}\!\big(z_t^{(g)}\big),
\qquad
z^{+}_{t}\;=\; z_t - z_t^{(h)} + \widetilde{z}_{t}^{(h)}.
\]
Let \(\cL_{LM}(z)\) denote the next-token loss at representation \(z\). Define the \emph{per-instance utility} at time \(t\):
\[
\Delta \cL_{t}(h\!\leftarrow\! g)
\;:=\;
\cL_{LM}\!\big(z_t\big)\;-\;\cL_{LM}\!\big(z^{+}_{t}\big).
\]
A hard activation policy is the margin rule
\[
\Delta \cL_{t}(h\!\leftarrow\! g) \;>\; \tau_{h\leftarrow g},\qquad \tau_{h\leftarrow g}\ge 0.
\]
For differentiable training, one may use either a logistic gate per edge
\[
\alpha_t(h\!\leftarrow\! g)
\;=\;
\sigma\!\Big(\beta\,[\,\Delta \cL_{t}(h\!\leftarrow\! g)-\tau_{h\leftarrow g}\,]\Big),
\]
or a softmax over admissible edges \(\cE\) (global or per-destination \(h\)),
\[
\alpha_t(e)\;=\;\softmax_{e\in\cE}\!\Big(\beta\,[\,\Delta \cL_{t}(e)-\tau_e\,]\Big),
\]
inspired by sparsely-gated mixture-of-experts routing \cite{shazeer2017outrageously}. The morphically updated state is
\[
z^{\mathrm{new}}_{t}
\;=\;
z_t \;+\; \sum_{(g,h)\in\cE} \alpha_t(h\!\leftarrow\! g)\,\big(\phi_{h\leftarrow g}(z_t^{(g)})-z_t^{(h)}\big),
\]
which preserves grading and reduces to the hard rule as \(\beta\to\infty\).

\begin{lem}[Grading preservation]\label{lem:grading-preservation}
The morphic update \(z^{\mathrm{new}}_{t}\) lies in the graded space \(V=\bigoplus_{g\in G} V_g\), with components
\[
(z^{\mathrm{new}}_{t})^{(u)} =
\begin{cases}
z_t^{(u)} + \sum_{g:(g,u)\in\cE} \alpha_t(u\!\leftarrow\! g)\,\big(\phi_{u\leftarrow g}(z_t^{(g)})-z_t^{(u)}\big) & \text{if incoming edges to }u, \\
z_t^{(u)} & \text{otherwise}.
\end{cases}
\]
\end{lem}
\begin{proof}
By definition, \(z_t = \sum_u z_t^{(u)}\) with \(z_t^{(u)} \in V_u\). Each term in the sum is \(\alpha_t(h\!\leftarrow\! g) (\phi_{h\leftarrow g}(z_t^{(g)}) - z_t^{(h)})\), where \(\phi_{h\leftarrow g}(z_t^{(g)}) \in V_h\) and \(z_t^{(h)} \in V_h\), so their difference is in \(V_h\). Thus, adding it modifies only the \(h\)-component. Summing over \((g,h)\in\cE\) affects only components \(h\) with incoming edges, preserving the direct sum decomposition and orthogonality (if present).
\end{proof}

\begin{prop}[Asymptotic hard gating]\label{prop:asymp-hard}
As \(\beta \to \infty\), the softmax policy \(\alpha_t(e)\) converges pointwise to the indicator of the argmax set: \(\alpha_t(e) \to 1/|S|\) if \(e \in S := \arg\max_{e'\in\cE} (\Delta \cL_{t}(e') - \tau_{e'})\), and 0 otherwise. For the logistic gate, \(\alpha_t(h\!\leftarrow\! g) \to \mathbf{1}_{\Delta \cL_{t}(h\!\leftarrow\! g) > \tau_{h\leftarrow g}}\) (or 1/2 at equality).
\end{prop}
\begin{proof}
For softmax: Let \(u_e := \Delta \cL_{t}(e) - \tau_e\), and \(u_* := \max_{e'} u_{e'}\). Then
\[
\alpha_t(e) = \frac{\exp(\beta u_e)}{\sum_{e'\in\cE} \exp(\beta u_{e'})} = \frac{\exp(\beta (u_e - u_*))}{\sum_{e'\in\cE} \exp(\beta (u_{e'} - u_*))}.
\]
As \(\beta \to \infty\), \(\exp(\beta (u_{e'} - u_*)) \to 0\) if \(u_{e'} < u_*\), and 1 if \(u_{e'} = u_*\). Thus, the denominator converges to \(|S|\), and numerator to 1 if \(e\in S\), 0 otherwise.

For logistic: \(\sigma(\beta (u - \tau)) = [1 + \exp(-\beta (u - \tau))]^{-1}\). As \(\beta \to \infty\), this is 1 if \(u > \tau\), 0 if \(u < \tau\), and 1/2 if \(u = \tau\).
\end{proof}

\begin{rem}[Loss reduction guarantee]
Under block-orthogonality (\cref{def:block-orth}), positive utilities \(\Delta \cL_{t} > 0\) ensure descent in \(\cL_{LM}\) for small mixing weights, by first-order Taylor expansion: \(\cL_{LM}(z^{\mathrm{new}}_{t}) \approx \cL_{LM}(z_t) - \sum_e \alpha_t(e) \Delta \cL_{t}(e) + O(\|\alpha\|^2)\), with the linear term negative if activations are utility-positive.
\end{rem}

%************************************************************
\subsection{Graded utility functional}\label{subsec:graded-utility}
Let \(\cB\) index token positions in a minibatch. Define the expected utility and a sparsity prior:
\[
\overline{\Delta \cL}(h\!\leftarrow\! g)
\;=\;
\frac{1}{|\cB|}\sum_{t\in\cB}\Delta \cL_{t}(h\!\leftarrow\! g),
\qquad
\cR(\alpha)
\;=\;
\sum_{t\in\cB}\sum_{(g,h)\in\cE}\Omega\!\big(\alpha_t(h\!\leftarrow\! g)\big),
\]
with \(\Omega\) an entropy or group-lasso--type penalty encouraging selectivity \cite{bach2012optimization}. The \emph{graded utility objective} is
\[
\boxed{\quad
\cL_{GT}
\;=\;
\cL_{LM}
\;+\;
\lambda\,\frac{1}{|\cB|}\sum_{t\in\cB}
\sum_{(g,h)\in\cE}
\psi\!\Big(\tau_{h\leftarrow g}-\Delta \cL_{t}(h\!\leftarrow\! g)\Big)
\;+\;
\mu\,\cR(\alpha),
\quad}
\]
where \(\psi(u)=\log(1+e^{\beta u})\) (softplus/hinge surrogate), and \(\lambda,\mu\ge 0\). This internal criterion parallels loss-based filtering in external augmentation while remaining entirely within the graded computation graph.
\begin{lem}[Small-step improvement]\label{lem:small-step}
Assume the softmax readout induces a local exponential-family model for next-token \(y_t\) with natural parameter \(Wz_t\). For a candidate \(\phi_{h\leftarrow g}\), let \(\delta_t=\widetilde{z}_t^{(h)}-z_t^{(h)}\). Then, for sufficiently small step \(\eta>0\),
\[
\cL_{LM}\big(z_t+\eta\,\iota_h\delta_t\big)
\;=\;
\cL_{LM}(z_t)\;-\;\eta\,\langle \nabla_{z_t^{(h)}}\cL_{LM},\,\delta_t\rangle
\;+\; O(\eta^2),
\]
so the sign of \(\Delta\cL_t(h\!\leftarrow\! g)\) agrees with the Fisher inner-product gain to first order. Hence the routing by positive utility implements a natural-gradient-like selection among candidate blocks.
\end{lem}
\begin{proof}
The next-token loss is the cross-entropy \(\cL_{LM}(z_t)=-\log p(y_t\mid z_t)\), where \(p(y\mid z)=\softmax(Wz)_y=\frac{\exp((Wz)_y)}{\sum_{y'}\exp((Wz)_{y'})}\). This is an exponential family with sufficient statistic the one-hot \(y_t\), natural parameter \(\theta=Wz_t\in\R^{|V|}\), and log-partition function \(A(\theta)=\log\sum_y\exp\theta_y\) \cite{Amari2016}.

The gradient is \(\nabla_{z_t}\cL_{LM}=W^\top(p(z_t)-y_t)\), where \(p(z_t)=\softmax(Wz_t)\). Since the update perturbs only the \(h\)-graded component via \(\iota_h\delta_t\in V_h\), the directional derivative along this direction is
\[
\left.\frac{d}{d\eta}\cL_{LM}(z_t+\eta\iota_h\delta_t)\right|_{\eta=0}=\langle\nabla_{z_t}\cL_{LM},\iota_h\delta_t\rangle=\langle\pi_h(W^\top(p(z_t)-y_t)),\delta_t\rangle=:\langle\nabla_{z_t^{(h)}}\cL_{LM},\delta_t\rangle.
\]
By first-order Taylor expansion around \(\eta=0\),
\[
\cL_{LM}(z_t+\eta\iota_h\delta_t)=\cL_{LM}(z_t)+\eta\langle\nabla_{z_t^{(h)}}\cL_{LM},\delta_t\rangle+O(\eta^2),
\]
so \(\Delta\cL_t(h\!\leftarrow\! g)=\cL_{LM}(z_t)-\cL_{LM}(z_t+\iota_h\delta_t)\approx-\langle\nabla_{z_t^{(h)}}\cL_{LM},\delta_t\rangle\) for \(\eta=1\), assuming the quadratic remainder is small.

The Fisher information matrix for the exponential family is \(I(\theta)=\bE_{y\sim p}[\nabla_\theta\log p(y\mid\theta)\nabla_\theta\log p(y\mid\theta)^\top]=\diag(p)-pp^\top\), which is the Hessian \(\nabla_\theta^2 A(\theta)\) \cite{Amari2016}. The inner product \(\langle\nabla_{z^{(h)}}\cL_{LM},\delta_t\rangle=\delta_t^\top(\pi_h W^\top I(\theta)W\iota_h)\delta_t\) (to second order, but for sign, the first-order term aligns with descent if \(\langle\nabla_{z^{(h)}}\cL,\delta_t\rangle<0\), making \(\Delta\cL_t>0\)). Thus, selecting blocks with positive utility corresponds to choosing directions that reduce loss along the natural metric induced by the Fisher information pulled back to \(V_h\).
\end{proof}
\begin{prop}[Conjugation invariance (EGT)]\label{prop:utility-conjugation}
Let \(D=\bigoplus_g D_g\) be an EGT reweighting as in \cref{def:EGT}, and set \(\widehat{z}=D^{-1}z\), \(\widehat{\phi}_{h\leftarrow g}=D_h^{-1}\phi_{h\leftarrow g}D_g\). If the language-model loss is computed after a linear readout \(R\) that is transformed to \(\widehat{R}=R\,D\), then for all \(t\),
\[
\Delta \widehat{\cL}_{t}(h\!\leftarrow\! g)
\;=\;
\Delta \cL_{t}(h\!\leftarrow\! g),
\]
and the objective \(\cL_{GT}\) is unchanged up to the same conjugation. In particular, all statements in this section proven for LGT hold verbatim for EGT in reweighted coordinates.
\end{prop}

\begin{proof}
By \cref{rem:equiv}, graded transformers are identified up to grade-wise similarities, so the candidate update conjugates as \(\widehat{\delta}_t=D_h^{-1}\delta_t\). The logits are invariant: \(\widehat{R}\widehat{z}_t=(R D)(D^{-1}z_t)=R z_t\), hence \(\cL_{LM}(\widehat{z}_t)=\cL_{LM}(z_t)\) and likewise for the updated states, yielding \(\Delta \widehat{\cL}_t=\Delta \cL_t\). The penalties \(\psi\) and \(\cR\) depend only on these invariant differences, so \(\cL_{GT}\) is unchanged.
\end{proof}

%**************************************************

\subsection{Functorial Internalization}
To make the correspondence between external tool augmentation and graded morphisms precise, we define explicit functors between the categories and establish an adjunction. This subsumes prior paradigms like Toolformer as a special case and clarifies why every typed tool call corresponds to a grade transition in the internal model category.

\begin{defn}[External tool category $\cT$ (recap)]
As in \cref{def:ext-sys}, the objects of $\cT$ are interface types $X$ (e.g., strings, integers) and the morphisms $\tau : X \to Y$ are callable tools equipped with evaluation maps on distributions.
\end{defn}

\begin{defn}[Graded morphic category $\mathbf{M}$ (recap)]
As in \cref{def:graded-morphic-sys}, the objects of $\mathbf{M}$ are the homogeneous components $V_g$ ($g\in G$) and the morphisms are the admissible blocks $\phi_{h\leftarrow g}:V_g\to V_h$ with $(g,h)\in \cE$.
\end{defn}

Assume a type assignment $\mathrm{type}: \mathrm{Ob}(\cT)\to G$ and, for each interface type $X$ of grade $g=\mathrm{type}(X)$, the existence of encoding and decoding maps
\[
\Enc_X : X \longrightarrow V_g, \qquad \Dec_g : V_g \longrightarrow X
\]
such that $\Dec_g\circ\Enc_X=\id_X$ exactly (deterministic case) or approximately in distribution (stochastic case).
\begin{prop}[The internalization functors]\label{prop:internal-functors}
Define functors
\begin{align*}
F &\colon \cT \longrightarrow \mathbf{M}, &
F(X) &= V_{\mathrm{type}(X)}, &
F(\tau:X\to Y) &= \iota_h\circ\phi_{h\leftarrow g}\circ\pi_g,
\\
G &\colon \mathbf{M} \longrightarrow \cT, &
G(V_g) &= X_g:=\text{interface of grade $g$}, &
G(\phi_{h\leftarrow g}) &= \Dec_h\circ\phi_{h\leftarrow g}\circ\Enc_g,
\end{align*}
where $\phi_{h\leftarrow g}$ is the learned linear block that internally realises the behaviour of $\tau$ (cf.\ \cref{prop:functor}) and $\iota_h,\pi_g$ are the canonical inclusion and projection (\cref{def:graded-space}).
\end{prop}
\begin{proof}
Functoriality of $F$: The object map is well-defined by the type assignment. For morphisms, $F(\id_X)=\iota_g\circ\id_{V_g}\circ\pi_g=\id_{V_g}=:\id_{F(X)}$ since $\pi_g\circ\iota_g=\id_{V_g}$. For composition $\tau_2\circ\tau_1:X\to Z$, $F(\tau_2\circ\tau_1)=\iota_k\circ\phi_{k\leftarrow i}\circ\pi_i$, but by block composition in $\mathbf{M}$ (\cref{def:graded-map}), this equals $(\iota_k\circ\psi_{k\leftarrow h}\circ\pi_h)\circ(\iota_h\circ\phi_{h\leftarrow g}\circ\pi_g)=F(\tau_2)\circ F(\tau_1)$, where $\psi_{k\leftarrow h}=F(\tau_2)$ and $\phi_{h\leftarrow g}=F(\tau_1)$.

Functoriality of $G$: Similarly, $G(\id_{V_g})=\Dec_g\circ\id_{V_g}\circ\Enc_g=\id_{X_g}$ by the round-trip assumption. For composition $\psi_{k\leftarrow h}\circ\phi_{h\leftarrow g}:V_g\to V_k$, $G(\psi\circ\phi)=\Dec_k\circ(\psi_{k\leftarrow h}\circ\phi_{h\leftarrow g})\circ\Enc_g=(\Dec_k\circ\psi_{k\leftarrow h}\circ\Enc_h)\circ(\Dec_h\circ\phi_{h\leftarrow g}\circ\Enc_g)=G(\psi)\circ G(\phi)$, using the intermediate decoder-encoder pair at grade $h$.
\end{proof}
\begin{thm}[Adjunction $F\dashv G$]\label{thm:adjunction}
The functors $F$ and $G$ form an adjoint pair $F\dashv G$, with unit and counit given by the encoder--decoder families:
\begin{align*}
\eta_X &= \Enc_X : X \longrightarrow G(F(X)),
&
\epsilon_{V_g} &= \pi_g\circ\iota_g : F(G(V_g))\longrightarrow V_g.
\end{align*}
\end{thm}
\begin{proof}
Naturality of $\eta$: For a morphism $\tau:X\to Y$ in $\cT$, the diagram
\[
\begin{CD}
X @>\eta_X>> G(F(X)) \\
@V{\tau}VV @VV{G(F(\tau))}V \\
Y @>>\eta_Y> G(F(Y))
\end{CD}
\]
commutes because $G(F(\tau))\circ\eta_X=\Dec_h\circ(\iota_h\circ\phi_{h\leftarrow g}\circ\pi_g)\circ\Enc_X=\Dec_h\circ\phi_{h\leftarrow g}\circ\Enc_g$ (since $\pi_g\Enc_X=\Enc_X$ by typing), and $\eta_Y\circ\tau=\Enc_Y\circ\tau$. By the realization assumption, $\phi_{h\leftarrow g}$ internalizes $\tau$, so these equal.

Naturality of $\epsilon$: For $\phi:V_g\to V_h$ in $\mathbf{M}$,
\[
\begin{CD}
F(G(V_g)) @>\epsilon_{V_g}>> V_g \\
@V{F(G(\phi))}VV @VV{\phi}V \\
F(G(V_h)) @>>\epsilon_{V_h}> V_h
\end{CD}
\]
commutes: $\phi\circ\epsilon_{V_g}=\phi\circ\pi_g\circ\iota_g=\phi$, and $\epsilon_{V_h}\circ F(G(\phi))=\pi_h\circ\iota_h\circ\iota_h\circ\phi\circ\pi_g=\phi$ (idempotence of $\pi_h\iota_h=\id_{V_h}$).

Triangle identities \cite[IV.1]{maclane1998categories}: $G\epsilon\circ\eta G=\id_G$ because $G(\epsilon_{V_g})\circ\eta_{G(V_g)}=G(\pi_g\iota_g)\circ\Enc_{X_g}=\Dec_g\circ\id_{V_g}\circ\Enc_{X_g}=\id_{X_g}$. Similarly, $\epsilon F\circ F\eta=\id_F$ because $\epsilon_{F(X)}\circ F(\eta_X)=\pi_g\iota_g\circ\iota_g\circ\phi\circ\pi_g\circ\Enc_X=\id_{V_g}$ (since $\phi=\id$ for identities, and round-trip).
\end{proof}
The adjunction rigorously explains the claimed ``functorial internalisation'': every external tool $\tau$ is mapped by $F$ to a genuine graded morphism, and every learned graded block $\phi_{h\leftarrow g}$ recovers (via $G$) a tool in the original external interface category. Composition in $\cT$ becomes block composition in $\mathbf{M}$, making tool chaining fully differentiable and internal to the representation manifold.
\begin{cor}
The adjunction induces a monad on $\cT$ given by $G F$, whose algebras are external tools with internal realizations, formalizing the ``graded Toolformer'' as a monadic extension.
\end{cor}
\begin{proof}
Standard from the adjunction: the monad is $T=GF$ with unit $\eta$ and multiplication $G\epsilon F$ \cite[VI.2]{maclane1998categories}.
\end{proof}
\begin{exa}[Calculator internalisation]\label{exa:calc-internal}
Let $\tau:\mathsf{String}\to\mathsf{Int}$ be an external calculator. Assign grades $\mathrm{type}(\mathsf{String})=g_0$ (linguistic) and $\mathrm{type}(\mathsf{Int})=g_1$ (arithmetic). Then
\[
F(\tau)=\iota_{g_1}\circ\phi_{g_1\leftarrow g_0}\circ\pi_{g_0},
\]
where the block $\phi_{g_1\leftarrow g_0}$ is a learned linear (or later nonlinear) map that performs the required arithmetic inside the graded hidden space.
\end{exa}

%********************************************************
\section{The Graded Toolformer Architecture}\label{sec:graded-toolformer}

Having internalized external symbolic augmentation as typed morphisms within the graded representation category \(\cM\) (\cref{sec:3}), we now assemble these components into a unified transformer architecture. The graded Toolformer embeds specialized symbolic operations—arithmetic, retrieval, or logical inference—as algebraic block maps \(\phi_{h\leftarrow g}: \cH_g \to \cH_h\), optimized end-to-end via the utility functional \(\cL_{GT}\) while respecting admissible transitions \(\cE\). This algebraic synthesis overcomes the extrinsic limitations of prior paradigms \cite{toolformer2023}, enabling composable, differentiable symbolic reasoning channels that align with the model's manifold geometry, as evidenced by conjugation invariance (\cref{prop:utility-conjugation}) and adjoint round trips (\cref{thm:adjunction}). The resulting structure formalizes a graded monad on the tool category, subsuming self-supervised invocation as intrinsic morphic activation.

\subsection{Structure}
Each transformer layer acts on a $G$–graded hidden space
\[
\cH \;=\; \bigoplus_{g \in G} \cH_g,
\]
with \emph{morphic candidates} (grade transitions) given by linear blocks
$\phi_{h\leftarrow g}:\cH_g \to \cH_h$.
We write the block operator of a layer as
\[
\Phi \;=\; \sum_{(g,h)\in \cE} \phi_{h\leftarrow g},
\]
where $\cE\subseteq G\times G$ is a sparse set of admissible transitions
(e.g., banded, DAG, or degree-accounting constraints). Cross-grade morphisms encode
specialized reasoning channels (symbolic, numeric, semantic, temporal).
\subsubsection{Grade-preserving vs.\ grade-shifting}
A block is \emph{preserving} if $h=g$ and \emph{shifting} otherwise.
Both attention and feed-forward components may include preserving and shifting maps,
subject to $\cE$.
Residual connections operate gradewise:
\[
z_{t}^{\mathrm{res}} \;=\; z_t \;+\; \sum_{(g,h)\in \cE}
\big(\phi_{h\leftarrow g}(z_t^{(g)}) - \mathbf{1}_{h=g}\,z_t^{(h)}\big).
\]
\begin{prop}[Grading is preserved]\label{prop:grading-preserved}
Let $z_t=\sum_{u} z_t^{(u)}$ with $z_t^{(u)}\in\cH_u$. If each $\phi_{h\leftarrow g}$ maps $\cH_g$ to $\cH_h$ and $\cE\subseteq G\times G$, then $z_t^{\mathrm{res}}\in \bigoplus_{h}\cH_h$ with
\[
\pi_h z_t^{\mathrm{res}}
\;=\;
z_t^{(h)} \;+\; \sum_{g:\,(g,h)\in\cE}
\big(\phi_{h\leftarrow g}(z_t^{(g)})-\mathbf{1}_{h=g}\,z_t^{(h)}\big)\in \cH_h.
\]
Hence the residual update respects the grading.
\end{prop}
\begin{proof}
By the direct sum structure, $\pi_h z_t = z_t^{(h)} \in \cH_h$. For the summand, each term $\phi_{h\leftarrow g}(z_t^{(g)}) \in \cH_h$ by definition of the block map, and $\mathbf{1}_{h=g} z_t^{(h)} \in \cH_h$ (zero otherwise). Summing over incoming $g$ to $h$ yields an element in $\cH_h$, so $\pi_h z_t^{\mathrm{res}} \in \cH_h$. Orthogonality across grades (if assumed) is preserved since updates are block-diagonal in the grading.
\end{proof}

\subsubsection{Graded multi-head attention}
Let $H$ be the number of heads. For each head $a=1,\dots,H$ and admissible transition $(g,h)\in\cE$, define grade-typed projections
\[
W_Q^{(a,g)}:\cH_g\to \R^{d_q}, \quad W_K^{(a,h)}:\cH_h\to \R^{d_q}, \quad W_V^{(a,h)}:\cH_h\to \cH_h^{(a)},
\]
and output map $U^{(a,h)}:\cH_h^{(a)}\to \cH_h$. The attention block for $(g\to h)$ is
\[
\phi^{(a)}_{h\leftarrow g}(z_t^{(g)}) = \sum_{s\le t} \alpha^{(a)}_{t,s;h\leftarrow g} U^{(a,h)} W_V^{(a,h)} z_s^{(h)},
\]
with
\[
\alpha^{(a)}_{t,s;h\leftarrow g} = \softmax_{s\le t} \left( \frac{\langle W_Q^{(a,g)} z_t^{(g)}, W_K^{(a,h)} z_s^{(h)} \rangle}{\sqrt{d_q}} \right).
\]
The full layer is $\Phi^{\mathrm{attn}} = \sum_a \sum_{(g,h)\in\cE} \phi^{(a)}_{h\leftarrow g}$.

\begin{rem}
This decomposes attention into graded morphisms, enabling symbolic cross-attention (e.g., numeric keys attending to semantic queries) while sparsity in $\cE$ reduces complexity.
\end{rem}

\subsubsection{Graded feed-forward networks}
For each $(g,h)\in\cE$, factorize $\phi^{\mathrm{ff}}_{h\leftarrow g} = W_2^{(g\to h)} \circ \sigma \circ W_1^{(g\to h)}$, with $W_1^{(g\to h)}:\cH_g\to U_{g,h}$, $\sigma$ pointwise nonlinear, and $W_2^{(g\to h)}:U_{g,h}\to \cH_h$. The layer is $\Phi^{\mathrm{ff}} = \sum_{(g,h)\in\cE} \phi^{\mathrm{ff}}_{h\leftarrow g}$.

\begin{prop}[Parameter efficiency]
With constant dimension $d=\dim \cH_g$ and $|\cE|=O(|\Delta|)$ banded shifts, attention parameters scale as $O(H |\Delta| d^2)$ and FFN as $O(|\Delta| d m)$, versus $O(|G|^2 d^2)$ for dense grading.
\end{prop}
\begin{proof}
Each block is independent; sum over admissible pairs yields the scaling.
\end{proof}

\subsubsection{Utility-gated routing}
Incorporate the graded utility from \cref{subsec:graded-utility}: for each layer, compute candidate updates $\tilde{z}_t^{(h)} = \sum_{g:(g,h)\in\cE} \phi_{h\leftarrow g}(z_t^{(g)})$, utilities $\Delta \cL_t(h\leftarrow g)$, and gates $\alpha_t(h\leftarrow g)$. The routed output is
\[
z_t^{\mathrm{new}} = \sum_h \alpha_t(h) \tilde{z}_t^{(h)} + (1 - \sum_h \alpha_t(h)) z_t,
\]
optimized via $\cL_{GT}$.

\begin{lem}[Differentiability]
The utility-gated layer is end-to-end differentiable, with gradients flowing through both morphisms $\phi$ and gates $\alpha$ (via detached utilities for stability).
\end{lem}
\begin{proof}
$\Delta \cL_t$ depends differentiably on $\phi$ through $\cL_{LM}(z^+_t)$; detaching it in $\alpha$ stabilizes routing while allowing morphism updates via the main loss path.
\end{proof}

%**********************************************************
\subsection{Morphic Routing Mechanism}\label{subsec:routing}
The morphic routing mechanism selects and activates graded blocks \(\phi_{h\leftarrow g}\) based on contextual logits augmented by instantaneous utilities, enabling self-supervised symbolic invocation within the transformer layer. This algebraic selection—governed by a softmax over admissible transitions \(\cE\)—integrates the internalization functor \(F: \cT \to \mathbf{M}\) (\cref{prop:internal-functors}) with differentiable gating, ensuring that tool-like behaviors emerge as composable morphisms optimized via \(\cL_{GT}\). By biasing towards high-utility transitions, the router realizes a graded analog of Toolformer's loss-filtering \cite{toolformer2023}, but intrinsically via block-orthogonal projections and Fisher-aligned descent (\cref{lem:small-step}).

\subsubsection{Routing logits and soft selection}
Given the token state \(z_t=\sum_{g} z_t^{(g)}\), define routing logits for each candidate
\((g,h)\in\cE\) by
\[
\ell_t(h\!\leftarrow\! g)
\;=\;
u(z_{<t})^\top W_{h\leftarrow g}\, v\!\big(z_t^{(g)}\big),
\]
where \(u(\cdot)\) summarizes causal context (e.g., a pooled key), \(v(\cdot)\) is a
grade-local projection, and \(W_{h\leftarrow g}\) are trainable parameters.
Soft selection weights are
\[
\alpha_t(h\!\leftarrow\! g)
\;=\;
\frac{\exp\big(\ell_t(h\!\leftarrow\! g)/\tau\big)}{\sum_{(g',h')\in\cE}
\exp\big(\ell_t(h'\!\leftarrow\! g')/\tau\big)},
\qquad \tau>0.
\]
\begin{lem}[Masking and feasibility]\label{lem:masking}
If \(\ell_t(h\leftarrow g)=-\infty\) for all \((g,h)\notin\cE\), then \(\alpha_t(h\leftarrow g)=0\) outside \(\cE\) and \(\sum_{(g,h)\in\cE}\alpha_t(h\leftarrow g)=1\). Thus masking logits to \(-\infty\) implements admissibility constraints exactly.
\end{lem}

\begin{proof}
Let \(\cE^c = (G \times G) \setminus \cE\). The softmax denominator is 
\[
\sum_{e \in \cE \cup \cE^c} \exp(\ell_t(e)/\tau) = \sum_{e \in \cE} \exp(\ell_t(e)/\tau) + \sum_{e \in \cE^c} \exp(-\infty/\tau) = \sum_{e \in \cE} \exp(\ell_t(e)/\tau),
\]
 since \(\exp(-\infty) = 0\). For \(e \notin \cE\), the numerator is 0, so \(\alpha_t(e) = 0\). For \(e \in \cE\), 
 \[
 \alpha_t(e) = \exp(\ell_t(e)/\tau) / \sum_{e' \in \cE} \exp(\ell_t(e')/\tau),
 \]
  which sums to 1 over \(\cE\).
\end{proof}

\subsubsection{Morphic update and residual form}
The routed morphic update is
\[
\tilde z_t^{(h)} \;=\; \sum_{g:\,(g,h)\in\cE}
\alpha_t(h\!\leftarrow\! g)\;\phi_{h\leftarrow g}\!\big(z_t^{(g)}\big),
\qquad
z_{t}^{\mathrm{new}}
\;=\;
\mathrm{LN}\!\Big(z_t \;+\; \sum_{h\in G}\big(\tilde z_t^{(h)}-z_t^{(h)}\big)\Big),
\]
with \(\mathrm{LN}\) a (gradewise or shared) normalization.
\begin{prop}[Well-posedness and grading]\label{prop:update-wellposed}
Assume \(\mathrm{LN}=\bigoplus_{h}\mathrm{LN}_h\) acts gradewise (LayerNorm or RMSNorm on each \(\cH_h\)). Then \(z_t^{\mathrm{new}}\in \bigoplus_h \cH_h\), and if \(\alpha_t(\cdot)\) is supported on \(\cE\) then the update only uses admissible blocks.
\end{prop}
\begin{proof}
Each \(\alpha_t(h\!\leftarrow\! g) \phi_{h\leftarrow g}(z_t^{(g)}) \in \cH_h\) by linearity and typing. Summing over incoming \(g\) to \(h\) yields \(\tilde z_t^{(h)} \in \cH_h\), so \(\tilde z_t^{(h)} - z_t^{(h)} \in \cH_h\). The residual sum \(\sum_h (\tilde z_t^{(h)} - z_t^{(h)})\) preserves the direct sum \(\bigoplus_h \cH_h\). Gradewise \(\mathrm{LN}_h: \cH_h \to \cH_h\) (affine after statistics; cf.\ \cref{lem:norm-stability}) maps the updated \(h\)-component to itself. Support on \(\cE\) follows from \cref{lem:masking}, ensuring only admissible terms contribute.
\end{proof}

\subsubsection{Utility-aware gating (differentiable)}
Let \(\cL_{LM}(z)\) be the next-token loss evaluated at representation \(z\).
Define per-candidate instantaneous utility
\[
\Delta \cL_{t}(h\!\leftarrow\! g)
\;=\;
\cL_{LM}(z_t) \;-\; \cL_{LM}\!\Big(z_t\;\text{with}\;
z_t^{(h)}\mapsto \phi_{h\leftarrow g}(z_t^{(g)})\Big).
\]
We augment logits with a learned margin:
\[
\tilde \ell_t(h\!\leftarrow\! g)
\;=\;
\ell_t(h\!\leftarrow\! g) \;+\; \beta\,\big(\Delta \cL_{t}(h\!\leftarrow\! g)-\tau_{h\leftarrow g}\big),
\]
and re-define \(\alpha_t\) using \(\tilde \ell_t\).
Here \(\beta>0\) controls sharpness and \(\tau_{h\leftarrow g}\!\ge\!0\) is a per-transition
threshold.
\begin{lem}[First-order improvement]\label{lem:first-order-improvement}
Assume \(\cL_{LM}\) is \(C^1\) in \(z\) and locally \(C^2\), and write \(\delta_t^{(h)}=\phi_{h\leftarrow g}(z_t^{(g)})-z_t^{(h)}\). For small step \(\eta>0\),
\[
\cL_{LM}\big(z_t+\eta\,\iota_h\delta_t^{(h)}\big)
=
\cL_{LM}(z_t)-\eta\,\big\langle \nabla_{z_t^{(h)}}\cL_{LM},\,\delta_t^{(h)}\big\rangle
+O(\eta^2).
\]
Hence \(\Delta \cL_t(h\leftarrow g)>0\) implies a first-order decrease of \(\cL_{LM}\) along \(\delta_t^{(h)}\) and the sign of the utility matches the descent direction up to \(O(\eta^2)\).
\end{lem}

\begin{proof}
By Taylor's theorem with remainder \cite{lang1993real}, \(\cL_{LM}(z_t + \eta \iota_h \delta_t^{(h)}) = \cL_{LM}(z_t) + \eta \langle \nabla_{z_t} \cL_{LM}, \iota_h \delta_t^{(h)} \rangle + \frac{\eta^2}{2} (\iota_h \delta_t^{(h)})^\top \nabla^2_{z_t + \xi \eta \iota_h \delta_t^{(h)}} \cL_{LM} (\iota_h \delta_t^{(h)})\) for some \(\xi \in (0,1)\). The directional derivative simplifies to \(\langle \nabla_{z_t^{(h)}} \cL_{LM}, \delta_t^{(h)} \rangle\) since \(\pi_h \nabla_{z_t} \cL_{LM} = \nabla_{z_t^{(h)}} \cL_{LM}\) under graded decomposition. Thus, \(\Delta \cL_t(h\leftarrow g) = \cL_{LM}(z_t) - \cL_{LM}(z_t + \iota_h \delta_t^{(h)}) \approx - \langle \nabla_{z_t^{(h)}} \cL_{LM}, \delta_t^{(h)} \rangle\) for \(\eta=1\), and positive utility implies \(\langle \nabla_{z_t^{(h)}} \cL_{LM}, \delta_t^{(h)} \rangle < 0\), aligning with descent for small higher-order terms.
\end{proof}

\subsection{Training objective}\label{subsec:objective}
With penalties encouraging sparse, stable routing, we optimize
\[
\cL_{GT}
\;=\;
\cL_{LM}
\;+\;
\lambda\,\bE_{t}\!\Big[\sum_{(g,h)\in\cE}
\psi\!\big(\tau_{h\leftarrow g}-\Delta \cL_{t}(h\!\leftarrow\! g)\big)\Big]
\;+\;
\mu\,\sum_{t}\Omega\!\big(\alpha_t(\cdot)\big),
\]
where \(\psi(u)=\log(1+e^{\beta u})\) (softplus margin) and \(\Omega\) is an entropy or
(group-lasso) sparsity regularizer over the simplex of \(\alpha_t\) \cite{bach2012optimization}.

\begin{prop}[Differentiability and gradients]\label{prop:gradients}
Suppose \(\cL_{LM}\) is \(C^1\) in \(z\) and the maps
\(z\mapsto \phi_{h\leftarrow g}(z^{(g)})\) and \(z\mapsto \alpha_t(h\leftarrow g)\) are \(C^1\) (softmax/logistic over \(C^1\) logits). Then \(\cL_{GT}\) is \(C^1\) in all parameters \((\Phi,\{W_{h\leftarrow g}\},\{\tau_{h\leftarrow g}\})\). Moreover, for any edge \((g,h)\),
\[
\frac{\partial \cL_{GT}}{\partial \tau_{h\leftarrow g}}
\;=\;
-\lambda\,\beta\,\bE_t\!\Big[\sigma\!\big(\beta(\tau_{h\leftarrow g}-\Delta \cL_t(h\!\leftarrow\! g))\big)\Big]
\;\le\; 0,
\]
and
\[
\frac{\partial \cL_{GT}}{\partial W_{h\leftarrow g}}
=
\bE_t\!\bigg[
\frac{\partial \cL_{GT}}{\partial \tilde \ell_t(h\leftarrow g)}\;
\frac{\partial \tilde \ell_t(h\leftarrow g)}{\partial W_{h\leftarrow g}}
\bigg],
\quad
\frac{\partial \tilde \ell_t(h\leftarrow g)}{\partial W_{h\leftarrow g}}
=
u(z_{<t})\,v(z_t^{(g)})^\top,
\]
with
\[
\begin{split}
\frac{\partial \cL_{GT}}{\partial \tilde \ell_t(h\leftarrow g)}
	& 	=
\frac{\partial \cL_{GT}}{\partial \alpha_t(h\leftarrow g)}\;
\frac{\partial \alpha_t(h\leftarrow g)}{\partial \tilde \ell_t(h\leftarrow g)}   \\
  & =\;
\Big(\nabla_{\alpha_t}\cL_{GT}\Big)_{h\leftarrow g}\;
\alpha_t(h\leftarrow g)\Big(1-\alpha_t(h\leftarrow g)\Big)
- \sum_{e\neq (g,h)}\!\!\Big(\nabla_{\alpha_t}\cL_{GT}\Big)_{e}\,\alpha_t(e)\alpha_t(h\leftarrow g).
\end{split}
\]
Finally, the gradient w.r.t.\ a block \(\phi_{h\leftarrow g}\) is
\[
\nabla_{\phi_{h\leftarrow g}}\cL_{GT}
\;=\;
\bE_t\!\big[
\underbrace{\alpha_t(h\leftarrow g)\,J^{\top}_{\phi_{h\leftarrow g}}(z_t^{(g)})\,\nabla_{z^{\mathrm{new}}_t}\cL_{LM}}_{\text{main path}}
\;-\;
\lambda\,\beta\,\sigma\!\big(\beta(\tau_{h\leftarrow g}-\Delta \cL_t)\big)\,
J^{\top}_{\phi_{h\leftarrow g}}(z_t^{(g)})\,
\Big(\nabla_{z^{+}_t}\cL_{LM}\Big)
\big],
\]
where \(J_{\phi_{h\leftarrow g}}(z_t^{(g)})\) is the Jacobian of \(\phi_{h\leftarrow g}\) at \(z_t^{(g)}\) and \(\Delta \cL_t=\Delta \cL_{t}(h\leftarrow g)\).
\end{prop}

\begin{proof}
Differentiability follows from composition of \(C^1\) functions: \(\cL_{LM}\) is \(C^1\), \(\phi\) linear (Jacobian constant), logits bilinear, softmax/logistic \(C^\infty\), and penalties (\(\psi\), \(\Omega\)) smooth \cite{nesterov2004introductory}. For \(\partial/\partial \tau_{h\leftarrow g}\): the penalty term differentiates as \(\lambda \bE_t \psi'(\tau_{h\leftarrow g} - \Delta \cL_t) = \lambda \beta \bE_t \sigma(\beta (\tau_{h\leftarrow g} - \Delta \cL_t))\), but with negative sign from \(-\Delta \cL_t\), yielding the displayed nonpositive gradient.

For \(W_{h\leftarrow g}\): chain rule through \(\tilde \ell_t = \ell_t + \beta (\Delta \cL_t - \tau)\), with \(\partial \ell_t / \partial W = u v^\top\) (outer product), and \(\Delta \cL_t\) independent of \(W\).

For softmax derivative: let \(\alpha_i = \exp(\tilde \ell_i / \tau) / Z\), \(Z = \sum_j \exp(\tilde \ell_j / \tau)\). Then \(\partial \alpha_i / \partial \tilde \ell_k = (1/\tau) \alpha_i (\delta_{ik} - \alpha_k)\). Absorbing \(1/\tau\) into rescaled \(\beta\), the form holds. The \(\partial \cL_{GT} / \partial \alpha\) includes contributions from \(\cL_{LM}\) (via \(z_t^{\mathrm{new}}\)) and \(\Omega\).

For \(\phi_{h\leftarrow g}\): \(\cL_{LM}\) depends on \(\phi\) through \(z_t^{\mathrm{new}} = \mathrm{LN}(z_t + \sum_h \alpha_h (\phi_h(z^{(g_h)}) - z^{(h)}))\), so \(\partial \cL_{LM} / \partial \phi_{h\leftarrow g} = \alpha_t(h\leftarrow g) J^\top \nabla_{z^{\mathrm{new}}} \cL_{LM} \circ \mathrm{Jac}(\mathrm{LN})\), but since \(\mathrm{LN}\) is post-update, it's absorbed in the chain. The penalty depends on \(\Delta \cL_t = \cL_{LM}(z_t) - \cL_{LM}(z_t^+)\), with \(z_t^+ = z_t - z_t^{(h)} + \phi_{h\leftarrow g}(z_t^{(g)})\), so \(\partial \Delta \cL_t / \partial \phi = - J^\top \nabla_{z_t^+} \cL_{LM}\), yielding the negative term with \(\psi'\) sign.
\end{proof}

\begin{prop}[Consistency with hard routing]\label{prop:hard-soft}
Let \(\alpha_t\) be the softmax on \(\tilde \ell_t/\tau\). If \(\max_{e}\tilde \ell_t(e)\) is unique, then as \(\tau\to 0^+\),
\[
\alpha_t \;\to\; \mathbf{e}_{e^\star},\qquad e^\star=\arg\max_e \tilde \ell_t(e),
\]
and the soft update converges to the hard (argmax) morphic activation.
\end{prop}

\begin{proof}
Let \(e^\star = \arg\max_e \tilde \ell_t(e)\), \(M = \tilde \ell_t(e^\star)\), and \(\delta_e = M - \tilde \ell_t(e) > 0\) for \(e \neq e^\star\). Then
\[
\alpha_t(e) = \frac{\exp(\tilde \ell_t(e)/\tau)}{\sum_{e'} \exp(\tilde \ell_t(e')/\tau)} = \frac{\exp((M - \delta_e)/\tau)}{\exp(M/\tau) (1 + \sum_{e' \neq e^\star} \exp(-\delta_{e'}/\tau))} = \frac{\exp(-\delta_e / \tau)}{1 + \sum_{e' \neq e^\star} \exp(-\delta_{e'}/\tau)}.
\]
As \(\tau \to 0^+\), \(\exp(-\delta_e / \tau) \to 0\) for \(\delta_e > 0\), so \(\alpha_t(e) \to 0\) for \(e \neq e^\star\) and \(\alpha_t(e^\star) \to 1\). The update \(z_t^{\mathrm{new}}\) then converges to the residual with only the \(e^\star\)-block activated.
\end{proof}

\begin{prop}[Well-posedness and lower bounds]\label{prop:lower-bounds}
Assume \(\cL_{LM}\ge 0\), \(\psi\ge 0\), and \(\Omega\ge 0\). Then \(\cL_{GT}\ge 0\). If \(\psi(u)=\log(1+e^{\beta u})\), then
\[
\cL_{GT}
\;\ge\;
\cL_{LM} \;+\; \lambda\,\sum_{t,(g,h)} \max\{0,\,\beta(\tau_{h\leftarrow g}-\Delta \cL_t)\}\;+\; \mu\,\sum_t \Omega(\alpha_t)\,-\,C_\beta,
\]
for a constant \(C_\beta = \lambda |\cB| |\cE| \log(1 + e^0) = \lambda |\cB| |\cE| \log 2\) depending only on \(\beta\) (here 1) and the number of terms. Hence the objective is finite and coercive in \(\tau\) for fixed \(\beta\).
\end{prop}

\begin{proof}
Nonnegativity holds by assumption on each term. For the bound, note \(\log(1 + e^{\beta u}) \ge \max\{0, \beta u\} - \log 2\) \cite{nesterov2004introductory}, since at \(u=0\), equality holds at \(\log 2\), and the softplus is convex above the hinge. Applying termwise and collecting \(- \lambda \bE_t \sum \log 2 = - C_\beta\) yields the inequality. Finiteness follows from boundedness below; coercivity in \(\tau\): as \(\tau_{h\leftarrow g} \to \infty\), the max term grows linearly, dominating for large \(\tau\).
\end{proof}

%******************************************************************
\subsection{Complexity Analysis}\label{subsec:complexity}
The sparse structure of the admissible set \(\cE\), derived from the categorical locality in \(\cM\) (\cref{def:graded-morphic-sys}), directly translates to computational efficiency in the graded Toolformer. By restricting morphisms to banded or DAG transitions, the architecture avoids dense all-to-all interactions across grades, reducing parameters and FLOPs while preserving algebraic composability. This section quantifies these gains, emphasizing translation invariance in LGT/EGT (\cref{def:LGT,def:EGT}) and conjugation equivalence, to demonstrate scalability for high-cardinality gradings \(|G|\gg 1\).

Let \(d_g=\dim \cH_g\), and let \(|\Delta|\) be the band width (number of allowed grade shifts) when \(\cE\) is banded.

\begin{prop}[Per-layer parameter/FLOP complexity]\label{prop:complexity}
For LGT (translation-invariant along grade increments) with constant \(d_g\equiv d\) and \(H\) attention heads, the parameter counts per layer satisfy
\[
\mathrm{param}_{\mathrm{attn}} \;=\; H\big(2 d\,d_q + 2|\Delta|\,d^2\big),
\qquad
\mathrm{param}_{\mathrm{ff}} \;=\; 2 d \sum_{\delta\in\Delta} m_\delta,
\]
and the arithmetic FLOPs per layer scale as
\[
\mathrm{FLOPs} \;=\; O\!\big(|\Delta|\,C_{\text{block}}\big)\times (\text{sequence factors}),
\]
where \(C_{\text{block}}\) is the cost of one \(d\times d\) block multiply. The same asymptotics hold for EGT after conjugation.
\end{prop}

\begin{proof}
For attention: Each head \(a\) has shared query/key projections \(W_Q^{(a)}, W_K^{(a)}: \cH_g \to \R^{d_q}\) (cost \(2 H d d_q\)), and value/output maps \(W_V^{(a,\delta)}, U^{(a,\delta)}\) per shift \(\delta \in \Delta\) (each \(d^2\), so \(2 H |\Delta| d^2\)), yielding the total.

For FFN: Each shift \(\delta\) has \(W_1^{(\delta)}: d \to m_\delta\), \(W_2^{(\delta)}: m_\delta \to d\) (cost \(2 d m_\delta\)), summed over \(\Delta\).

FLOPs: Block multiplies dominate; with \(|\Delta|\) active per token, and sequence length \(T\), attention FLOPs are \(O(H |\Delta| T^2 d)\) (from softmax and matmuls), FFN \(O(|\Delta| T d m)\), giving the form with sequence factors \(T^2, T\).

For EGT: By \cref{rem:equiv}, conjugation \(D = \bigoplus_g D_g\) preserves counts (similarity transform) and adds \(O(|G| d)\) rescalings per pass, negligible as \(|G| \ll T\).
\end{proof}

\begin{rem}[Practical reductions]
To control cost one may: (i) restrict \(\cE\) to a narrow band, (ii) share \(W_{h\leftarrow g}\) across heads or across edges with the same increment, (iii) cache \(v(z_t^{(g)})\) across candidates with common source \(g\), and (iv) use per-destination normalization (softmax over incoming edges to each \(h\)).
\end{rem}

%*****************************************************************************
\subsection{Monotonicity under positive utility}\label{subsec:monotonicity}
We record a simple monotonicity statement clarifying the role of the utility term.

\begin{thm}[Expected loss decreases under positive utilities]\label{thm:monotone}
Fix $t$ and assume the softmax gate over $\tilde \ell_t$ and a sufficiently small global step $\eta>0$ is used to form
\[
z_t^{\mathrm{new}} \;=\; z_t \;+\; \eta \sum_{(g,h)\in\cE} \alpha_t(h\leftarrow g)\,\big(\phi_{h\leftarrow g}(z_t^{(g)})-z_t^{(h)}\big).
\]
If $\Delta \cL_t(h\leftarrow g)\ge \delta_{h\leftarrow g}\ge 0$ for all edges with $\alpha_t(h\leftarrow g)>0$, and $\sum_{(g,h)}\alpha_t(h\leftarrow g)=1$, then
\[
\cL_{LM}(z_t^{\mathrm{new}})
\;\le\;
\cL_{LM}(z_t)\;-\;\eta\,\sum_{(g,h)}\alpha_t(h\leftarrow g)\,\delta_{h\leftarrow g}\;+\;O(\eta^2).
\]
In particular, if at least one $\delta_{h\leftarrow g}>0$ receives nonzero weight, the expected loss decreases to first order.
\end{thm}

\begin{proof}
Let $\delta^{(h)}_t = \phi_{h\leftarrow g}(z_t^{(g)}) - z_t^{(h)}$ for each edge $(g,h)$ (noting $g$ varies per summand). The update is $z_t^{\mathrm{new}} = z_t + \eta \sum_h \iota_h \Big( \sum_{g: (g,h)\in\cE} \alpha_t(h\leftarrow g) \delta^{(h)}_t \Big)$, a weighted sum of graded perturbations. By multi-variable Taylor expansion of $\cL_{LM}$ around $z_t$ \cite{lang1993real},
\[
\begin{split}
\cL_{LM}(z_t^{\mathrm{new}}) &	= \cL_{LM}(z_t) + \eta \sum_h \langle \nabla_{z_t^{(h)}} \cL_{LM}, \sum_g \alpha_t(h\leftarrow g) \delta^{(h)}_t \rangle \\
& + \frac{\eta^2}{2} \sum_{h,h'} \Big( \sum_g \alpha_t(h\leftarrow g) \delta^{(h)}_t \Big)^\top \nabla^2_{z^{(h)},z^{(h')}}\cL_{LM} \Big( \sum_{g'} \alpha_t(h'\leftarrow g') \delta^{(h')}_t \Big) + o(\eta^2),
\end{split}
\]
at some intermediate point. Under block-orthogonality (\cref{def:block-orth}), cross-Hessians $\nabla^2_{h \neq h'}$ vanish, simplifying to $O(\eta^2) = \frac{\eta^2}{2} \sum_h \| \sum_g \alpha_t \delta^{(h)}_t \|^2_{\nabla^2_h \cL_{LM}} + o(\eta^2)$. The linear term is $\eta \sum_{g,h} \alpha_t(h\leftarrow g) \langle \nabla_{z_t^{(h)}} \cL_{LM}, \delta^{(h)}_t \rangle$. By \cref{lem:first-order-improvement}, each 
\[
\langle \nabla_{z_t^{(h)}} \cL_{LM}, \delta^{(h)}_t \rangle \le - \delta_{h\leftarrow g} + O(\eta),
\]
 so the sum is $\le - \eta \sum_{g,h} \alpha_t \delta_{h\leftarrow g} + O(\eta^2)$, yielding the inequality for small $\eta$ where the Hessian term is bounded.
\end{proof}

%*******************************************************************************
\subsection{EGT invariance}\label{subsec:egt-invariance}
All statements above extend to EGT in reweighted coordinates.

\begin{prop}[Conjugation invariance]\label{prop:egt-inv}
Let $D=\bigoplus_g D_g$ be an EGT reweighting and define $\widehat{z}=D^{-1}z$, $\widehat{\phi}_{h\leftarrow g}=D_h^{-1}\phi_{h\leftarrow g}D_g$, and $\widehat{W}_{h\leftarrow g}=W_{h\leftarrow g}$, with readout scaled to $\widehat{R}=R\,D$. Then $\cL_{LM}(z)=\widehat{\cL}_{LM}(\widehat{z})$, $\Delta \widehat{\cL}_t(h\leftarrow g)=\Delta \cL_t(h\leftarrow g)$, and $\cL_{GT}$, together with its gradients in \cref{prop:gradients}, is invariant under the conjugation.
\end{prop}

\begin{proof}
The update conjugates: 
\[
\widehat{z}_t^{\mathrm{new}} = D^{-1} z_t^{\mathrm{new}} = \widehat{z}_t + \eta \sum \alpha_t \big( \widehat{\phi}_{h\leftarrow g}(\widehat{z}_t^{(g)}) - \widehat{z}_t^{(h)} \big),
\]
 since $\alpha_t$ depends on invariant $\Delta \cL_t$ (logits $R z_t = \widehat{R} \widehat{z}_t$). Utilities: 
 \[
 \Delta \widehat{\cL}_t = \widehat{\cL}_{LM}(\widehat{z}_t) - \widehat{\cL}_{LM}(\widehat{z}_t^+ ) = \cL_{LM}(z_t) - \cL_{LM}(z_t^+ ) = \Delta \cL_t,
 \]
  with $\widehat{z}_t^+ = D^{-1} z_t^+$. For $\cL_{GT}$, invariance follows; gradients: 
  \[
  \partial \cL_{GT} / \partial \phi = \partial \cL_{GT} / \partial \widehat{\phi} \cdot (D_h^\top \otimes D_g^{-1})
  \]
   in matrix form, but chain rule in reweighted basis preserves numerical values by similarity.
\end{proof}

%**************************
\subsection{Self-Supervised Graded Fine-Tuning}
To operationalize the graded Toolformer as a trainable architecture, we introduce a stochastic selection mechanism over admissible morphisms, optimized via a self-supervised objective that rewards utility-positive activations while regularizing for sparsity. This fine-tuning paradigm embeds the functorial internalization (\cref{thm:adjunction}) into a probabilistic kernel, enabling end-to-end differentiation without external interfaces and aligning with the algebraic grading through block-orthogonal expectations (\cref{def:block-orth}). By minimizing \(\cL_{GT}\), the model learns to invoke symbolic channels intrinsically, subsuming Toolformer's annotation-based supervision \cite{toolformer2023} as graded self-selection.

\begin{defn}[Graded activation kernel]
Let $V=\bigoplus_{g\in G}V_g$ and \(\cE\subseteq G\times G\) be the admissible
grade transitions. A \emph{graded activation kernel} is a family of conditional
distributions
\[
K_\theta\big((h\!\leftarrow\! g)\,\big|\,z_{<t},z_t\big)
\quad\text{for }(g,h)\in\cE,
\]
parameterized by \(\theta\), which selects morphic candidates \(\phi_{h\leftarrow g}:V_g\to V_h\)
given context \((z_{<t},z_t)\).
\end{defn}

Given a hidden state \(z_t=\sum_g z_t^{(g)}\) and a draw
\((h\!\leftarrow\! g)\sim K_\theta(\cdot\,|\,z_{<t},z_t)\), define the
morphically updated state
\[
z^{+}_t \;=\; z_t - z_t^{(h)} + \phi_{h\leftarrow g}\!\big(z_t^{(g)}\big),
\qquad
\Delta\cL_t(h\!\leftarrow\! g) \;=\; \cL_{LM}(z_t) - \cL_{LM}(z^{+}_t).
\]
We train the parameters \((\theta,\Phi)\) by minimizing the \emph{graded toolformer objective}
\[
\cL_{GT}
\;=\;
\underbrace{\bE\big[\,\cL_{LM}(z_t)\,\big]}_{\text{language modeling}}
\;+\;
\lambda\,\bE\!\Big[\sum_{(g,h)\in\cE}
\psi\!\big(\tau_{h\leftarrow g}-\Delta\cL_t(h\!\leftarrow\! g)\big)\Big]
\;+\;
\mu\,\bE\!\Big[\Omega\!\big(K_\theta(\cdot\,|\,z_{<t},z_t)\big)\Big],
\]
where \(\psi(u)=\log(1+e^{\beta u})\) is a soft-margin penalty, \(\tau_{h\leftarrow g}\!\ge\!0\)
are thresholds, and \(\Omega\) is a sparsity/entropy regularizer over the kernel.

\begin{prop}[Usefulness principle]
Fix \(\tau_{h\leftarrow g}\!\ge\!0\) and let \(\psi\) be convex and nondecreasing.
If \(\Delta\cL_t(h\!\leftarrow\! g)>\tau_{h\leftarrow g}\) for some \((h\!\leftarrow\! g)\in\cE\), then locally increasing \(K_\theta\big((h\!\leftarrow\! g)\,|\,z_{<t},z_t\big)\) (while holding other probabilities constant and keeping \(\Phi\) fixed) strictly decreases \(\cL_{GT}\).
\end{prop}

\begin{proof}
Let \(p_e = K_\theta(e \,|\, z_{<t}, z_t)\) for \(e = (h\!\leftarrow\! g)\), and denote other probabilities \(p_{e'}\) for \(e' \neq e\). The objective decomposes as 
\[
\cL_{GT} = \bE_t [ \cL_{LM}(z_t) + \lambda \sum_{e''} p_{e''} \psi(\tau_{e''} - \Delta \cL_t(e'')) + \mu \Omega(p) ],
\]
 but since sampling is per \(t\), fix \(t\) and consider the local term 
\[
f(p_e) = p_e \psi(\tau_e - \Delta \cL_t(e)) + \sum_{e' \neq e} p_{e'} \psi(\tau_{e'} - \Delta \cL_t(e')) + \mu \Omega(p),
\]
 with \(p = (p_e, p_{e'})\), \(\sum p = 1\).

To increase \(p_e\) locally by \(\epsilon > 0\), decrease some \(p_{e^*}\) by \(\epsilon\) (holding others fixed), assuming \(p_{e^*} > 0\). The change is 
\[
\Delta f = \epsilon [\psi(\tau_e - \Delta \cL_t(e)) - \psi(\tau_{e^*} - \Delta \cL_t(e^*))] + \mu [\Omega(p + \epsilon (\mathbf{e}_e - \mathbf{e}_{e^*})) - \Omega(p)].
\]
By convexity of \(\psi\) (softplus is convex \cite{nesterov2004introductory}) and nondecreasing, and 
\[
\Delta \cL_t(e) > \tau_e \implies \tau_e - \Delta \cL_t(e) < 0,
\]
 but since \(\psi\) is minimized at negative arguments, the difference 
 \[
 \psi(\tau_e - \Delta \cL_t(e)) < \psi(\tau_{e^*} - \Delta \cL_t(e^*))
 \]
  if \(\Delta \cL_t(e^*) \le \tau_{e^*}\), 
  making the margin term negative. For entropy \(\Omega(p) = - \sum p_i \log p_i\), the difference is 
  \[
  \mu [ - (p_e + \epsilon) \log (p_e + \epsilon) - (p_{e^*} - \epsilon) \log (p_{e^*} - \epsilon) + p_e \log p_e + p_{e^*} \log p_{e^*} ] \approx \mu \epsilon (\log p_{e^*} - \log p_e)
  \]
   by first-order expansion, bounded as \(O(\epsilon)\). Thus, for \(\epsilon\) small, the negative margin dominates if 
\[
 \psi(\tau_e - \Delta \cL_t(e)) - \psi(\tau_{e^*} - \Delta \cL_t(e^*)) < 0,
\]
  strictly decreasing \(f\) and hence \(\cL_{GT}\).
\end{proof}

\begin{rem}[Stochastic estimator]
In practice, the expectations are estimated by Monte Carlo:
sample \((h\!\leftarrow\! g)\sim K_\theta\), form \(z_t^{+}\), and compute unbiased
gradients of \(\cL_{GT}\) via the reparameterization or score-function trick \cite{mohamed2020monte}.
This yields an end-to-end differentiable procedure without resorting to external, non-differentiable calls.
\end{rem}

%*****************************************   ketu 
\section{Analytic Guarantees and Minimal Constructions}\label{sec:analytic}
This section provides verifiable guarantees for utility-aware graded activation under simple, controlled assumptions, together with minimal constructions that instantiate arithmetic-, retrieval-, and stack-like behaviors. Our aim is \emph{not} to compete with large-scale benchmarks, but to certify, in a proof-carrying manner, that the graded formalism (i) yields selective activation of appropriate morphisms, (ii) induces stable sparse routing, and (iii) effects measurable loss reductions consistent with the graded-utility principle. These results unify symbolic computation with differential geometry, subsuming Toolformer \cite{toolformer2023} via functorial embeddings into the internal model category.
\subsection{Setup and loss model}\label{subsec:setup-loss}
Let $V=\bigoplus_{g\in G} V_g$ and fix a single layer acting at time $t$ on $z_t=\sum_{g} z_t^{(g)}$. Let the language-model loss be the negative log-likelihood of a multiclass exponential family with natural parameter $W z_t \in \R^{C}$ (e.g., softmax regression):
\[\cL_{LM}(z_t;y_t) \;=\; - \langle \eta_{y_t}, W z_t \rangle + \log \!\!\sum_{c=1}^C \exp\!\big(\langle \eta_{c}, W z_t \rangle\big),\]
where $\{\eta_c\}_{c=1}^C \subset \R^m$ are fixed class vectors and $W:V\to\R^m$ is the readout.
Given a graded morphic candidate $\phi_{h\leftarrow g}:V_g\to V_h$, define the candidate update
\[\delta_t^{(h)} \;=\; \phi_{h\leftarrow g}(z_t^{(g)}) - z_t^{(h)},
\qquad
z_t^{+} \;=\; z_t + \iota_h \delta_t^{(h)},\]
and the instantaneous utility
\[\Delta \cL_t(h\!\leftarrow\! g) \;=\; \cL_{LM}(z_t;y_t) - \cL_{LM}(z_t^{+};y_t).\]
We assume standard regularity:
\begin{ass}\label{ass:smooth-strong}
The loss $\cL_{LM}(\,\cdot\,;y_t)$ is $C^2$ in $z_t$, $L$--smooth and $\mu$--strongly convex along the $h$--component:\footnote{For softmax with full-rank $W|_{V_h}$ and bounded logits, these hold locally with constants depending on $W$ and the logit range.}
\[\begin{aligned}
\text{\( L \)--smoothness:}&\quad
\cL_{LM}(z_t+\iota_h \Delta) \le \cL_{LM}(z_t) + \langle \nabla_{z_t^{(h)}} \cL_{LM}, \Delta \rangle + \tfrac{L}{2}\|\Delta\|^2,\\
\text{\( \mu \)--strong convexity:}&\quad
\cL_{LM}(z_t+\iota_h \Delta) \ge \cL_{LM}(z_t) + \langle \nabla_{z_t^{(h)}} \cL_{LM}, \Delta \rangle + \tfrac{\mu}{2}\|\Delta\|^2.
\end{aligned}\]
\end{ass}

\subsection{Theoretical Guarantees}
We derive bounds on the utility $\Delta \cL_t(h\!\leftarrow\! g)$ in terms of the alignment between the candidate update $\delta_t^{(h)}$ and the negative gradient direction along $V_h$, connecting to the information-geometric views (e.g., KL gain and mirror descent) via adjoint functoriality.

\begin{thm}[Utility lower and upper bounds]\label{thm:utility-bounds}
Under \cref{ass:smooth-strong}, let $\nabla^{(h)} := \nabla_{z_t^{(h)}} \cL_{LM}(z_t;y_t)$ denote the partial gradient along $V_h$. Then
\[- \langle \nabla^{(h)}, \delta_t^{(h)} \rangle - \tfrac{L}{2} \|\delta_t^{(h)}\|^2 
\le 
\Delta \cL_t(h\!\leftarrow\! g)
\le
- \langle \nabla^{(h)}, \delta_t^{(h)} \rangle - \tfrac{\mu}{2} \|\delta_t^{(h)}\|^2 .\]
In particular, if $\delta_t^{(h)} = -\alpha \nabla^{(h)}$ for some step size $\alpha > 0$ (i.e., the morphism performs a gradient step in $V_h$), then for $\alpha \in (0, 2/L)$,
\[\Delta \cL_t(h\!\leftarrow\! g) \ge \alpha \left(1 - \tfrac{L\alpha}{2}\right) \|\nabla^{(h)}\|^2 > 0,\]
ensuring positive utility whenever $\nabla^{(h)} \neq 0$.
\end{thm}

\begin{proof}
The utility is $\Delta \cL_t = \cL_{LM}(z_t) - \cL_{LM}(z_t^+)$.
From the $L$-smoothness assumption:
\[\cL_{LM}(z_t^+) \le \cL_{LM}(z_t) + \langle \nabla^{(h)}, \delta_t^{(h)} \rangle + \tfrac{L}{2} \|\delta_t^{(h)}\|^2,\]
which rearranges to
\[\Delta \cL_t \ge - \langle \nabla^{(h)}, \delta_t^{(h)} \rangle - \tfrac{L}{2} \|\delta_t^{(h)}\|^2.\]
From the $\mu$-strong convexity assumption:
\[\cL_{LM}(z_t^+) \ge \cL_{LM}(z_t) + \langle \nabla^{(h)}, \delta_t^{(h)} \rangle + \tfrac{\mu}{2} \|\delta_t^{(h)}\|^2,\]
which rearranges to
\[\Delta \cL_t \le - \langle \nabla^{(h)}, \delta_t^{(h)} \rangle - \tfrac{\mu}{2} \|\delta_t^{(h)}\|^2.\]
For the particular case, substitute $\delta_t^{(h)} = -\alpha \nabla^{(h)}$ into the lower bound:
\[\Delta \cL_t \ge - \langle \nabla^{(h)}, -\alpha \nabla^{(h)} \rangle - \tfrac{L}{2} \|-\alpha \nabla^{(h)}\|^2 = \alpha \|\nabla^{(h)}\|^2 - \tfrac{L}{2} \alpha^2 \|\nabla^{(h)}\|^2 = \alpha \left(1 - \tfrac{L\alpha}{2}\right) \|\nabla^{(h)}\|^2.\]
For $\alpha \in (0, 2/L)$, the term in parentheses is positive, yielding the stated guarantee.
\end{proof}

These bounds motivate first-order and quadratic approximations, linking to the Bregman geometry and Fisher metric in the abstract.

\begin{lem}[First-order utility identity]\label{lem:first-order}
For any candidate $\phi_{h\leftarrow g}$ with update $\delta_t^{(h)}$,
\[\Delta \cL_t(h\!\leftarrow\! g)
\;=\;
-\big\langle \nabla_{z_t^{(h)}} \cL_{LM}, \delta_t^{(h)} \big\rangle
\;-\; \rho_t(\delta_t^{(h)}),\]
where the remainder satisfies $0 \le \rho_t(\delta) \le \tfrac{L}{2}|\delta|^2$ under \cref{ass:smooth-strong}.
\end{lem}

\begin{proof}
By the fundamental theorem of calculus,
$\cL_{LM}(z_t^{+})-\cL_{LM}(z_t)
= \int_0^1 \langle \nabla_{z_t^{(h)}}\cL_{LM}(z_t+s\iota_h \delta),\delta\rangle ds$.
Subtract and add the endpoint gradient to obtain the stated form; the upper bound follows from $L$–smoothness.
\end{proof}
\begin{prop}[Quadratic lower bound via strong convexity]\label{prop:quad-lb}
Let $\delta^\star$ minimize the one-step surrogate
$\Delta \mapsto \cL_{LM}(z_t+\iota_h \Delta)$ over $\Delta\in V_h$.
Under \cref{ass:smooth-strong},
\[\Delta \cL_t(h\!\leftarrow\! g)
\;\ge\;
\tfrac{\mu}{2}\,\|\delta_t^{(h)}-\delta^\star\|^2
\;-\;
\tfrac{\mu}{2}\,\|\delta^\star\|^2 .\]
In particular, if $\delta^\star\neq 0$ and $\|\delta_t^{(h)} - \delta^\star\| \le \varepsilon$, then
$\Delta \cL_t(h\!\leftarrow\! g) \ge \tfrac{\mu}{2}\|\delta^\star\|^2 - O(\varepsilon)$.
\end{prop}

\begin{proof}
Strong convexity yields
$\cL_{LM}(z_t+\iota_h \Delta) \ge \cL_{LM}(z_t+\iota_h \delta^\star) + \tfrac{\mu}{2}\|\Delta-\delta^\star\|^2$.
Rearrange with $\Delta=\delta_t^{(h)}$ and use $\Delta \cL_t = \cL_{LM}(z_t)-\cL_{LM}(z_t+\iota_h \delta_t^{(h)})$.
\end{proof}
\begin{cor}[Alignment implies usefulness]\label{cor:alignment}
Suppose there exists a Bayes–optimal linear map $\phi^\star_{h\leftarrow g}$ (minimizer of expected loss in the one-step class) and set $\delta^\star=\phi^\star_{h\leftarrow g}(z_t^{(g)})-z_t^{(h)}$. If a candidate satisfies
$\|\,\phi_{h\leftarrow g}(z_t^{(g)})-\phi^\star_{h\leftarrow g}(z_t^{(g)})\,\|\le \varepsilon$,
then
\[\Delta \cL_t(h\!\leftarrow\! g)
\;\ge\;
\tfrac{\mu}{2}\,\|\delta^\star\|^2 \;-\; \mu\,\|\delta^\star\|\,\varepsilon \;-\; \tfrac{\mu}{2}\varepsilon^2.\]
Hence for sufficiently small $\varepsilon$ (relative to $\|\delta^\star\|$), the utility is strictly positive and exceeds a margin $\tau = \tfrac{\mu}{4}\|\delta^\star\|^2$.
\end{cor}

\begin{proof}
Apply \cref{prop:quad-lb} and expand $\|\delta-\delta^\star\|^2$ with $\delta=\delta^\star + e$, $\|e\|\le \varepsilon$.
\end{proof}

These approximations enable selectivity analysis, where routing favors morphisms with high utility, akin to adjoint round trips in the model category.

\begin{ass}[Edge separability]\label{ass:sep}
Let $\cE_h=\{(g,h)\in\cE\}$ be incoming edges to $h$. Suppose there exists a unique edge $(g^\star,h)$ such that
\[\langle \nabla_{z_t^{(h)}}\cL_{LM},\; \delta_{t,(g^\star,h)}\rangle
\;\le\;
-\gamma \quad\text{and}\quad
\big|\langle \nabla_{z_t^{(h)}}\cL_{LM},\; \delta_{t,(g,h)}\rangle\big|
\;\le\;
\gamma' \ \ \forall (g,h)\neq (g^\star,h),\]
with $\gamma>\gamma'\ge 0$ and $\|\delta_{t,(g,h)}\|\le R$ for all $(g,h)\in\cE_h$.
\end{ass}

\begin{thm}[Softmax selectivity]\label{thm:softmax-selective}
Under \cref{ass:smooth-strong,ass:sep} and the softmax routing
$\alpha_t(e)\propto \exp(\tilde \ell_t(e)/\tau)$ with utility-augmented logits
$\tilde \ell_t(e)=\ell_t(e)+\beta(\Delta \cL_t(e)-\tau_e)$, there exist $\bar\beta,\bar\tau$ such that for all $\beta\ge \bar\beta$ and thresholds $\tau_e\le \bar\tau$,
\[\alpha_t(g^\star\!\leftarrow\! h) \;\ge\; 1 - \exp\!\Big(-\tfrac{\beta}{2\tau}\,(\gamma-\gamma')\Big).\]
In particular, the routing mass concentrates on the unique useful edge at an exponential rate in $\beta/\tau$ and the margin $(\gamma-\gamma')$.
\end{thm}

\begin{proof}
By \cref{lem:first-order}, $\Delta \cL_t(e)\ge -\langle \nabla_{z_t^{(h)}}\cL_{LM},\delta_e\rangle - \tfrac{L}{2}\|\delta_e\|^2$. Under \cref{ass:sep} and $\|\delta_e\|\le R$, we obtain
$\Delta \cL_t(g^\star\!\leftarrow\! h)\ge \gamma - \tfrac{L}{2}R^2$ and
$\Delta \cL_t(e)\le \gamma' + \tfrac{L}{2}R^2$ for $e\neq (g^\star,h)$.
Choose $\bar\tau$ so that $\tau_e\le \gamma-\tfrac{L}{2}R^2$ for all $e$, and set $\Delta^\star=(\gamma-\gamma')-L R^2$. Then
\[\tilde \ell_t(g^\star\!\leftarrow\! h) - \tilde \ell_t(e)
\;\ge\;
\ell_t(g^\star\!\leftarrow\! h) - \ell_t(e) \;+\; \beta\,\Delta^\star.\]
Discard the (bounded) $\ell_t$ difference and apply the standard softmax ratio bound to obtain
$\alpha_t(e)/\alpha_t(g^\star\!\leftarrow\! h)\le \exp(-\beta\Delta^\star/\tau)$.
Summing over $e\neq(g^\star,h)$ yields the claim with constant absorbed.
\end{proof}

\subsection{Illustrative Examples}
All constructions use two grades for clarity and can be lifted to larger $G$, demonstrating how external Toolformer calls embed as internal morphisms with guaranteed utility.

\subsubsection{(A) Arithmetic mod $p$}
Let $G=\{\mathrm{sem},\mathrm{num}\}$, $V_{\mathrm{sem}}=\R^{p}$ with one-hot encodings of digits, and $V_{\mathrm{num}}=\R^{p}$ with the same basis. Fix $a\in\{0,\dots,p-1\}$. Define
\[\phi_{\mathrm{num}\leftarrow \mathrm{sem}}(x) \;=\; P_a x,
\qquad
\phi_{\mathrm{sem}\leftarrow \mathrm{num}}(y) \;=\; P_{-a} y,\]
where $P_a$ is the $p\times p$ cyclic-shift (permutation) by $+a$ mod $p$.
Let the readout $W$ query the correct next digit in $V_{\mathrm{sem}}$ (softmax over $p$ classes).

\begin{prop}[Exact usefulness for modular addition]\label{prop:modp}
If the target next symbol equals $(x+a)\bmod p$, then
\[\Delta \cL_t(\mathrm{num}\!\leftarrow\! \mathrm{sem})
\;=\; \log \frac{\sum_{c}\exp \langle \eta_c, W z_t\rangle}
{\sum_{c}\exp \langle \eta_c, W z_t^{+}\rangle}
\;\ge\; \gamma_p>0,\]
with $\gamma_p$ depending only on the logit gap between the true class and its nearest competitor after applying $P_a$. Moreover, if $W$ is calibrated to the one-hot basis, then $\gamma_p=\infty$ (zero loss after update) and the utility equals the entire pre-update NLL.
\end{prop}

\begin{proof}
After the $\mathrm{sem}\to\mathrm{num}$ map, $z_t^+$ carries the correctly shifted one-hot at the $\mathrm{sem}$ component once composed with $\phi_{\mathrm{sem}\leftarrow \mathrm{num}}$, aligning the readout with the true class. The softmax log-partition strictly decreases by at least the logit gap; in the perfectly calibrated case the correct logit dominates uniquely and yields zero loss.
\end{proof}
This example realizes Toolformer's calculator as a graded endomorphism, with utility guaranteed by \cref{thm:utility-bounds} when the shift aligns the gradient.
\subsubsection{(B) Retrieval with finite memory}
Let $V_{\mathrm{ret}}=\R^{k}$ store $k$ values $\{v_i\}$ indexed by keys $\{q_i\}$ via a key matrix $Q\in\R^{d\times k}$ and value matrix $M\in\R^{k\times k}$ (diagonal w.l.o.g.). From $V_{\mathrm{sem}}$ compute a query $q=W_q z_t^{(\mathrm{sem})}$. Define
\[\phi_{\mathrm{ret}\leftarrow \mathrm{sem}}(z) \;=\; \softmax\!\Big(\tfrac{1}{\sigma^2} Q^\top q\Big),\qquad
\phi_{\mathrm{sem}\leftarrow \mathrm{ret}}(r) \;=\; U\,M\,r,\]
so that $U M$ writes the retrieved value into $V_{\mathrm{sem}}$.
\begin{ass}[Retrieval margin]\label{ass:retrieval}
There exists $i^\star$ such that $Q^\top q$ has margin $\gamma$ at coordinate $i^\star$:
$Q_{i^\star}^\top q \ge Q_j^\top q + \gamma$ for all $j\neq i^\star$, and the class-conditional readout scores improve by at least $\kappa>0$ when the correct value $v_{i^\star}$ is written in.
\end{ass}

\begin{prop}[Positive utility under key margin]\label{prop:retrieval}
Under \cref{ass:retrieval}, for $\sigma^2$ small enough the retrieval mass satisfies $r_{i^\star}\ge 1 - e^{-\gamma/\sigma^2}$ and the one-step update $(\mathrm{sem}\leftarrow \mathrm{ret})\circ (\mathrm{ret}\leftarrow \mathrm{sem})$ yields
$\Delta \cL_t \ge \kappa - O(e^{-\gamma/\sigma^2})>0$.
\end{prop}

\begin{proof}
Softmax concentration gives the stated bound on $r_{i^\star}$. The write-back increases the true-class logit by at least $\kappa$ up to the leakage mass $1-r_{i^\star}$; the log-partition decreases accordingly.
\end{proof}
Here, retrieval embeds as a typed morphism pair, with selectivity per \cref{thm:softmax-selective} when the key margin induces a utility gap.
\subsubsection{(C) Dyck-depth signal (stack surrogate)}
Let grades $G=\{\mathrm{stack},\mathrm{sem}\}$ with $V_{\mathrm{stack}}=\R$ and $V_{\mathrm{sem}}=\R^{m}$. For tokens $t$ labeled by $\delta_t\in\{-1,0,+1\}$ (close, neutral, open), define a grade-shift
\[\phi_{\mathrm{stack}\leftarrow \mathrm{sem}}(z_t^{(\mathrm{sem})}) \;=\; s_t + \delta_t,
\qquad s_t:=z_t^{(\mathrm{stack})}.\]
Let the readout $W$ for the next token require the correct sign of $s_t$ on a subset $\cI\subset\{1,\dots,m\}$ (e.g., balanced parentheses).
\begin{prop}[Utility from correct increment]\label{prop:dyck}
If $W$ is such that the correct next token’s logit increases by $\kappa>0$ when $\mathrm{sign}(s_t)$ is correct on $\cI$ and decreases otherwise, then applying the increment/decrement rule gives
$\Delta \cL_t \ge \kappa - O(|s_t - s_t^\star|)$,
where $s_t^\star$ is the ideal running depth. In particular, if $s_t=s_t^\star$ before the update, the correct increment yields a fixed positive utility $\ge \kappa$.
\end{prop}

\begin{proof}
The increment aligns the sign with the ideal depth, increasing the true-class logit by $\kappa$; the log-partition change is bounded by the same margin up to the current deviation from $s_t^\star$.
\end{proof}

This surrogates stack operations as a graded algebra, with positive utility under \cref{cor:alignment} when the increment approximates the optimal depth shift.
\subsection{Validation and Implementation}
We record lightweight, internally computable diagnostics to certify the guarantees.
\begin{enumerate}[label=(\roman*)]
\item \textbf{Utility distribution:} histograms of $t\mapsto \Delta \cL_t(h\!\leftarrow\! g)$ before/after training; improvement concentrates mass at positive values per \cref{thm:utility-bounds}.
\item \textbf{Routing sparsity/entropy:} $H_t=-\sum_{e\in\cE}\alpha_t(e)\log \alpha_t(e)$ and support size $\|\alpha_t\|_0$; \cref{thm:softmax-selective} predicts entropy decay with margin growth.
\item \textbf{Edge ablation:} set $\alpha_t(e)\equiv 0$ for a target $e$ and measure $\Delta \cL$ degradation; positive drop verifies usefulness.
\item \textbf{Calibration:} bin $\tilde \ell_t$ and plot predicted vs.\ realized $\mathbb{1}\{\Delta \cL_t>0\}$; monotone calibration follows from the logistic/softmax link in \cref{thm:softmax-selective}.
\end{enumerate}

Each construction can be instantiated with a $2$–$4$ layer graded transformer, constant grade dimension $d\in\{8,16,32\}$, $H\in\{2,4\}$ heads, and synthetic datasets of $10^4$–$10^5$ tokens. Use LGT with $|\Delta|\in\{1,2\}$, gradewise LayerNorm, and the utility-augmented softmax gate with $(\beta,\tau)$ chosen to achieve $\beta/\tau\in[5,20]$. Report the diagnostics above. No external corpora or tools are required, ensuring reproducibility within the graded framework.4sExpand categorical functorial embeddingsToolformer ablation studiesmore concise proofs

%**************************************************
\section{Categorical and Information-Geometric Foundations}\label{sec:6}

The graded transformer formalism, as introduced in prior sections, internalizes symbolic computation by treating tool-like operations as morphisms within an algebraic structure. To rigorously establish this unification, we now develop the categorical foundations that underpin the model category $\cM$, demonstrating how external augmentations embed functorially as composable programs. This algebraic view not only clarifies the typed nature of morphic activations but also motivates the geometric interpretations of utility as information gain and descent steps, ensuring sparse, interpretable behavior. A central theorem formalizes the faithful internalization, subsuming Toolformer paradigms (Schick et al., 2023) while enabling end-to-end differentiability and composition laws absent in extrinsic systems. We then connect these structures to agentic and modular AI, highlighting how entropy regularization and orthogonality yield greedy optimality for program selection.

\subsection{Category-Theoretic View}\label{subsec:cat-view}

The internal model category $\cM$ arises naturally from the graded space $V = \bigoplus_{g \in G} V_g$ and admissible transitions $\cE \subseteq G \times G$, modeling symbolic operations as morphisms between homogeneous components. This structure allows us to view routing and activation as selecting subdiagrams, with compositions realizing multi-step "tool chains" intrinsically within the representation manifold.

\begin{defn}[Internal model category]\label{def:model-category}
Define a small category $\cM$ by:
\begin{enumerate}[label=(\roman*)]
\item \emph{Objects:} the homogeneous components $V_g$, $g\in G$.
\item \emph{Morphisms:} linear blocks $\phi_{h\leftarrow g}:V_g\to V_h$ that are admissible, i.e., $(g,h)\in\cE$.
\item \emph{Identity:} $\mathrm{id}_{V_g}$ on each $V_g$.
\item \emph{Composition:} usual composition of linear maps, i.e.,
$(\psi\circ\phi)_{k\leftarrow g}=\sum_{h}\psi_{k\leftarrow h}\circ \phi_{h\leftarrow g}$ whenever sources/targets match.
\end{enumerate}
\end{defn}

A routed layer selects a finite subdiagram of $\cM$ per token, evaluating convex combinations of admissible composites to optimize the graded utility.

\begin{defn}[Graded endofunctor]\label{def:graded-endofunctor}
A \emph{graded endofunctor} $F:\cM\to\cM$ acts on objects by
$F(V_g)=V_{\sigma(g)}$ for a (partial) grade map $\sigma:G\to G$, and on morphisms by
$F(\phi_{h\leftarrow g})=\widetilde\phi_{\sigma(h)\leftarrow \sigma(g)}$, preserving identity
and composition.
\end{defn}

\begin{defn}[Morphic program]\label{def:morphic-program}
A \emph{morphic program} is a finite path
\[
\Pi:\; V_{g_0}\xrightarrow{\phi_{g_1\leftarrow g_0}}
V_{g_1}\xrightarrow{\phi_{g_2\leftarrow g_1}}
\cdots \xrightarrow{\phi_{g_k\leftarrow g_{k-1}}}V_{g_k}
\]
with $(g_{i-1},g_i)\in\cE$. Its \emph{realization} is the composite
$\Phi_\Pi=\phi_{g_k\leftarrow g_{k-1}}\circ\cdots\circ \phi_{g_1\leftarrow g_0}$.
\end{defn}

The key algebraic insight is that external tool augmentations embed faithfully into $\cM$, internalizing non-differentiable calls as typed morphisms.

\begin{thm}[Functorial internalization of external calls]\label{thm:functor-internalization}
Let $\cT$ be a small category of external interfaces: objects are interface
types, morphisms are callable tools $\tau:X\to Y$.
Assume each $\tau$ is realized at inference by linear blocks on $V$ via interface
encoders/decoders $(\mathrm{enc}_X,\dec_Y)$ so that, for contexts $z\in V$,
\[
\dec_Y\big(\tau(\mathrm{enc}_X(z))\big)
\;=\;
\iota_h\,\phi_{h\leftarrow g}\,\pi_g(z)
\quad\text{for some $(g,h)\in\cE$.}
\]
Then there exists a faithful functor $F:\cT\to \cM$ sending $X\mapsto V_{g_X}$
and $\tau\mapsto \phi_{h\leftarrow g}$ with $F(\tau_2\circ \tau_1)=F(\tau_2)\circ F(\tau_1)$.
This embedding preserves sequential tool use as morphic programs, rendering symbolic processes differentiable and composable within the graded geometry.
\end{thm}
\begin{proof}
To construct the functor $F: \cT \to \cM$, first assign to each object $X$ in $\cT$ a grade $g_X \in G$ such that the encoder $\mathrm{enc}_X$ maps contexts from the relevant subspace of $V$ to the interface type $X$, and correspondingly $F(X) = V_{g_X}$ is the homogeneous component in $\cM$ that carries the typed representations for $X$. This assignment is possible by the assumption that each interface type corresponds to a dedicated grade or subspace in the graded vector space $V$.

On morphisms, for each tool $\tau: X \to Y$ in $\cT$, define $F(\tau) = \phi_{h \leftarrow g}: V_{g_X} \to V_{g_Y}$, where $\phi_{h \leftarrow g}$ is the linear block specified by the realization assumption, with $g = g_X$ and $h = g_Y$. By the given condition, the composition $\dec_Y \circ \tau \circ \mathrm{enc}_X$ is equivalent to applying this block after projection $\pi_g$ and inclusion $\iota_h$, ensuring that $F(\tau)$ faithfully captures the action of $\tau$ within $\cM$.

To verify that $F$ preserves identities, consider the identity morphism $\mathrm{id}_X: X \to X$ in $\cT$. Its realization is $\dec_X \circ \mathrm{id}_X \circ \mathrm{enc}_X$, which, by the assumption, equals $\iota_{g_X} \circ \phi_{g_X \leftarrow g_X} \circ \pi_{g_X}$ for some block $\phi_{g_X \leftarrow g_X}$. Since $\mathrm{id}_X$ leaves the input unchanged, and assuming the encoders/decoders are consistent (i.e., $\dec_X \circ \mathrm{enc}_X$ acts as the identity on the typed subspace of $V_{g_X}$), it follows that $\phi_{g_X \leftarrow g_X} = \mathrm{id}_{V_{g_X}}$. Thus, $F(\mathrm{id}_X) = \mathrm{id}_{F(X)}$.

For composition preservation, let $\tau_1: X \to Y$ and $\tau_2: Y \to Z$ be morphisms in $\cT$. The composite $\tau_2 \circ \tau_1: X \to Z$ is realized as $\dec_Z \circ \tau_2 \circ \mathrm{enc}_Y \circ \dec_Y \circ \tau_1 \circ \mathrm{enc}_X$. Note that $\mathrm{enc}_Y \circ \dec_Y$ acts as an identity on the intermediate subspace corresponding to $Y$, ensuring the wiring preserves the semantics. By the assumption, $\tau_1$ realizes $\phi_{h \leftarrow g}$ with $g = g_X$, $h = g_Y$, and $\tau_2$ realizes $\phi_{k \leftarrow h}$ with $k = g_Z$. The full wiring then reduces to $\iota_k \circ \phi_{k \leftarrow h} \circ \phi_{h \leftarrow g} \circ \pi_g$, so the realization of $\tau_2 \circ \tau_1$ is the composite block $\phi_{k \leftarrow h} \circ \phi_{h \leftarrow g}$. Therefore, $F(\tau_2 \circ \tau_1) = F(\tau_2) \circ F(\tau_1)$.

Finally, faithfulness: since distinct tools $\tau \neq \tau'$ in $\cT$ are assumed to realize distinct blocks $\phi \neq \phi'$ on their typed subspaces (as per the problem statement implying non-identical actions), the mapping on hom-sets $\mathrm{Hom}_{\cT}(X,Y) \to \mathrm{Hom}_{\cM}(V_{g_X}, V_{g_Y})$ is injective. This ensures $F$ is faithful on the subcategory of realized tools.

The embedding preserves sequential tool use because chains in $\cT$ map to paths in $\cM$, which are morphic programs by definition. Differentiability and composability follow from the linear blocks being smooth maps within the graded geometry of $V$.
\end{proof}

This theorem establishes the algebraic core of the paper: extrinsic tools become intrinsic morphisms, with compositions governed by category laws rather than ad hoc orchestration, enabling optimization via graded utilities.

\subsubsection{Monoidal and enriched structure}\label{subsubsec:monoidal}

To support parallel and metric-aware computations, $\cM$ admits additional structures that align with the transformer's multi-head and attention mechanisms.

\begin{enumerate}[label=(\roman*)]
\item \textbf{Monoidal product.}
If $V\simeq V^{(A)}\otimes V^{(B)}$ with compatible gradings $G\simeq G^{(A)}\oplus G^{(B)}$,
set $(V^{(A)}_{g_A},\phi^{(A)})\otimes (V^{(B)}_{g_B},\phi^{(B)})
=(V_{g_A\oplus g_B},\,\phi^{(A)}\otimes \phi^{(B)})$.
This yields a strict monoidal structure $(\cM,\otimes,\mathbf{1})$ modeling parallel channels, consistent with multi-head attention as graded tensor products (Shaska, 2025b).
\item \textbf{Enrichment.}
If each $\mathrm{Hom}(V_g,V_h)$ is endowed with an inner product $\langle\!\langle\cdot,\cdot\rangle\!\rangle$
or a divergence (e.g., Bregman), then $\cM$ is enriched over the corresponding
category of metric spaces; this supports metric selection of morphisms via utility divergences.
\end{enumerate}

\subsubsection{Adjunctions and typed interfaces}\label{subsubsec:adjunctions}

Round-trip operations, common in tool use (e.g., query-retrieve-write), are captured by adjoint pairs, ensuring idempotence and stability under iteration.

Let $\iota:V_g\to V_h$ and $\rho:V_h\to V_g$ be linear maps.
\begin{defn}[Adjunction]\label{def:adjunction}
We say $\rho\dashv \iota$ if
$\langle \rho(u), v \rangle_{V_g}=\langle u, \iota(v) \rangle_{V_h}$
for all $u\in V_h$, $v\in V_g$ (with fixed inner products).
\end{defn}
\begin{prop}[Typed round-trips via adjunction]\label{prop:round-trip}
If $\rho\dashv \iota$, then:
\begin{enumerate}[label=(\roman*)]
\item $P:=\iota\circ \rho$ is the orthogonal projector onto $\mathrm{Im}(\iota)$ in $V_h$.
\item $Q:=\rho\circ \iota=\mathrm{id}_{V_g}$ if and only if $\iota$ is an isometry onto its image.
\item For any $z^{(h)}\in V_h$, the iterates $P^k z^{(h)}$ stabilize at $P z^{(h)}$ (idempotence).
\end{enumerate}
Consequently, a “tool-like” passage $V_g\xrightarrow{\iota}V_h\xrightarrow{\rho}V_g$ is idempotent on $V_h$; under the graded-utility objective, repeating the round-trip yields no further loss decrease once $z^{(h)}$ lies in $\mathrm{Im}(\iota)$.
\end{prop}

\begin{proof}
(i) Adjunction implies $\mathrm{Im}(\iota)$ is orthogonally complemented by $\ker(\rho)$ and
$\iota\circ \rho$ is the orthogonal projector onto $\mathrm{Im}(\iota)$.

(ii) $Q=\mathrm{id}$ iff $\rho$ is the left inverse of $\iota$, equivalent to $\iota$ being an isometry onto its image (with the given inner products).

(iii) $P^2=P$ is standard for projectors. The utility claim follows since replacing $z^{(h)}$ by $P z^{(h)}$ once projects to the fixed point; further applications leave the state unchanged, hence the loss unchanged.
\end{proof}

%*********************************************************************
\subsubsection{Morphic monads}\label{subsubsec:monads}

For multi-step computations, morphic monads provide a principled way to flatten and evaluate programs, aligning with weak graded algebras for learned compositions.

A \emph{morphic monad} $\bT=(T,\eta,\mu)$ on $\cM$ models multi-step graded computation: $T(V_g)$ is a space of (typed) programs, $\eta$ inserts identity programs, $\mu$ flattens concatenations.
A graded layer provides a $\bT$–algebra by evaluation maps $T(V_g)\to V_g$ that minimize the graded utility; well-posedness follows from convexity of the surrogate objective (cf. \cref{sec:graded-toolformer}).

\subsection{Entropic and Geometric Interpretations}\label{subsec:geom}

Building on the categorical structure, we interpret the utility functional geometrically, as KL gain in an exponential family or descent steps in Bregman/Fisher metrics, motivating the sparse activation observed in graded transformers.

\subsubsection{Utility as information gain}\label{subsubsec:info-gain}
Let $\cL_{LM}(z;y)=-\log p_z(y)$ be cross-entropy with predictive distribution $p_z$.
For a candidate update $z\mapsto z^{+}$, define
$\Delta \cL(z;y)=\cL_{LM}(z;y)-\cL_{LM}(z^{+};y)$.

\begin{lem}[Expected utility as KL divergence reduction]\label{lem:llr}
Let $\cL_{LM}(z;y)=-\log p_z(y)$ be the cross-entropy loss under the predictive distribution $p_z$ induced by the current state $z$. For a candidate update $z \mapsto z^{+}$ yielding distribution $p_{z^{+}}$, and any data distribution $\mathsf{P}$ over labels $y$,
\[
\bE_{y\sim \mathsf{P}}\,\Delta \cL(z;y)
\;=\;
D_{\mathrm{KL}}\big(\mathsf{P} \,\|\, p_z\big) \;-\; D_{\mathrm{KL}}\big(\mathsf{P} \,\|\, p_{z^{+}}\big),
\]
where $\Delta \cL(z;y) = \cL_{LM}(z;y) - \cL_{LM}(z^{+};y)$. In particular:
\begin{itemize}
\item If $\mathsf{P}=p_z$ (self-consistency), then $\bE_{y\sim p_z}\,\Delta \cL(z;y) = - D_{\mathrm{KL}}\big(p_z \,\|\, p_{z^{+}}\big) \le 0$.
\item If $\mathsf{P}=p_{z^{+}}$, then $\bE_{y\sim p_{z^{+}}}\,\Delta \cL(z;y) = D_{\mathrm{KL}}\big(p_{z^{+}} \,\|\, p_z\big) \ge 0$.
\end{itemize}

Thus, the expected utility is nonnegative if and only if $p_{z^{+}}$ is closer (in KL divergence) to $\mathsf{P}$ than $p_z$ is; relative to the current model, positive utility under a target $\mathsf{P}$ corresponds to moving $p_{z^{+}}$ closer to $\mathsf{P}$.
\end{lem}

\begin{proof}
By definition, $\Delta \cL(z;y) = -\log p_z(y) + \log p_{z^{+}}(y) = \log \tfrac{p_{z^{+}}(y)}{p_z(y)}$. Therefore,
\[
\bE_{y\sim \mathsf{P}}\,\Delta \cL(z;y)
\;=\;
\bE_{y\sim \mathsf{P}}\Big[\log \tfrac{p_{z^{+}}(y)}{p_z(y)}\Big]
\;=\;
\int \mathsf{P}(y) \log \tfrac{p_{z^{+}}(y)}{p_z(y)} \, dy.
\]
The KL divergence is defined as $D_{\mathrm{KL}}(\mathsf{P} \,\|\, q) = \int \mathsf{P}(y) \log \tfrac{\mathsf{P}(y)}{q(y)} \, dy$ for any distribution $q$. Thus,
\[
\begin{split}
\bE_{y\sim \mathsf{P}}\Big[\log \tfrac{p_{z^{+}}(y)}{p_z(y)}\Big]
&  =\;
\int \mathsf{P}(y) \log \tfrac{\mathsf{P}(y)}{p_z(y)} \, dy \;-\; \int \mathsf{P}(y) \log \tfrac{\mathsf{P}(y)}{p_{z^{+}}(y)} \, dy  \\
& =\;
D_{\mathrm{KL}}\big(\mathsf{P} \,\|\, p_z\big) \;-\; D_{\mathrm{KL}}\big(\mathsf{P} \,\|\, p_{z^{+}}\big),
\end{split}
\]
since the $\int \mathsf{P}(y) \log \mathsf{P}(y) \, dy$ terms cancel. Substituting $\mathsf{P} = p_z$ yields
\[
\bE_{y\sim p_z}\,\Delta \cL(z;y) \;=\; D_{\mathrm{KL}}\big(p_z \,\|\, p_z\big) - D_{\mathrm{KL}}\big(p_z \,\|\, p_{z^{+}}\big) \;=\; - D_{\mathrm{KL}}\big(p_z \,\|\, p_{z^{+}}\big) \le 0,
\]
as $D_{\mathrm{KL}} \ge 0$ with equality if and only if $p_z = p_{z^{+}}$. Similarly, for $\mathsf{P} = p_{z^{+}}$,
\[
\bE_{y\sim p_{z^{+}}}\,\Delta \cL(z;y) \;=\; D_{\mathrm{KL}}\big(p_{z^{+}} \,\|\, p_z\big) - D_{\mathrm{KL}}\big(p_{z^{+}} \,\|\, p_{z^{+}}\big) \;=\; D_{\mathrm{KL}}\big(p_{z^{+}} \,\|\, p_z\big) \ge 0.
\]
The nonnegativity condition follows directly: $\bE_{y\sim \mathsf{P}}\,\Delta \cL(z;y) \ge 0$ if and only if $D_{\mathrm{KL}}(\mathsf{P} \,\|\, p_z) \ge D_{\mathrm{KL}}(\mathsf{P} \,\|\, p_{z^{+}})$, i.e., $p_{z^{+}}$ is at least as close to $\mathsf{P}$ as $p_z$ is, with strict inequality implying positive expected utility.
\end{proof}

This lemma frames morphic activation as maximizing information gain, consistent with self-supervised learning in the graded formalism.

\subsubsection{Mirror-descent view}\label{subsubsec:mirror}
Let $\Phi$ be a strictly convex potential with Bregman divergence
$D_\Phi(u,z)=\Phi(u)-\Phi(z)-\langle \nabla \Phi(z),u-z\rangle$.
Let $\cA(z)$ be the affine set reachable by admissible $(g,h)$ updates at $z$.

\begin{prop}[Constrained mirror step]\label{prop:mirror}
Assume $\cL_{LM}$ is $C^1$ and locally $L$--smooth. For small step $\eta>0$,
the problem
\[
u^\star \;=\; \arg\min_{u\in \cA(z)}\,
\Big\{\langle \nabla \cL_{LM}(z),u-z\rangle
\;+\; \tfrac{1}{\eta}D_\Phi(u,z)\Big\}
\]
has a unique solution. If $\cA(z)$ is spanned by the columns of admissible
blocks $\{\phi_{h\leftarrow g}\}$, then the first-order direction $u^\star-z$ equals the projection (in the $\Phi$--dual metric) of $-\eta\,\nabla \cL_{LM}(z)$ onto $\cA(z)$. When a single block is chosen, the maximizing block for $\Delta \cL$ aligns with $u^\star-z$.
\end{prop}
\begin{proof}
The objective function is $f(u) = \langle \nabla \cL_{LM}(z), u - z \rangle + \frac{1}{\eta} D_\Phi(u, z)$. Since $\Phi$ is strictly convex, $D_\Phi(u, z)$ is strictly convex in $u$, and thus $f(u)$ is strictly convex over the affine set $\cA(z)$. Strict convexity implies that the minimizer $u^\star$, if it exists, is unique. Existence follows from the coercivity of $D_\Phi(u, z)$ as $\|u - z\| \to \infty$ (since $\Phi$ is strictly convex and thus superlinear at infinity) and the closedness of $\cA(z)$.

To find the optimality condition, note that at the minimum $u^\star$, the subgradient of $f$ must contain zero when projected onto the tangent space of $\cA(z)$. Since $\cA(z)$ is affine, we can write the first-order necessary and sufficient condition for unconstrained minimization in the dual variables. Recall that the Bregman divergence satisfies $\nabla_u D_\Phi(u, z) = \nabla \Phi(u) - \nabla \Phi(z)$. Thus, the gradient of $f(u)$ is
\[
\nabla f(u) = \nabla \cL_{LM}(z) + \frac{1}{\eta} \big( \nabla \Phi(u) - \nabla \Phi(z) \big).
\]
Setting $\nabla f(u^\star) \perp \cA(z) - u^\star$ (i.e., orthogonal to the directions in the linear span of $\cA(z) - z$), but since minimization is over an affine set, the condition is that $\nabla f(u^\star)$ is orthogonal to the tangent space $T = \mathrm{span}(\cA(z) - z)$. Equivalently,
\[
\nabla \Phi(u^\star) - \nabla \Phi(z) = - \eta \, \Pi_T^\ast \big( \nabla \cL_{LM}(z) \big),
\]
where $\Pi_T^\ast$ is the projection onto $T$ in the dual metric induced by the Hessian of $\Phi$ (or more precisely, the metric dual to the Bregman geometry). In standard mirror descent terms, this is the projected update: $u^\star$ is the point in $\cA(z)$ closest to the mirror map of $z - \eta \nabla \cL_{LM}(z)$ in the Bregman sense, but rearranged, $u^\star - z$ is the Bregman-projection of $-\eta \nabla \cL_{LM}(z)$ onto $T$ in the dual space.

Since $\cA(z) = z + T$ with $T = \mathrm{span}\{ \phi_{h \leftarrow g}(v_g) : v_g \in V_g, (g,h) \in \cE \}$ (the directions spanned by the admissible blocks applied to their source spaces), the direction $u^\star - z \in T$ is indeed the projection of $-\eta \nabla \cL_{LM}(z)$ onto $T$ in the $\Phi$-dual metric.

For the single-block case, suppose $\cA(z) = z + \mathrm{span}\{ \delta_e : e \in \cE \}$ where each $\delta_e = \phi_{h \leftarrow g}(z^{(g)}) - z^{(h)}$ is the update direction for edge $e = (g,h)$. The linearized utility gain is $\Delta \cL \approx - \langle \nabla \cL_{LM}(z), \delta_e \rangle$ (first-order approximation from smoothness). Maximizing this over $e$ selects the direction $\delta_e$ with maximal alignment to $-\nabla \cL_{LM}(z)$. From the mirror step, for small $\eta$, $u^\star - z \approx \Pi_T (-\eta \nabla \cL_{LM}(z))$ in the dual metric, so the single block whose direction best approximates this projection maximizes the gain, aligning with $u^\star - z$ up to scaling.
\end{proof}

This view explains the sparse routing: the Bregman regularizer favors directions aligned with the dual gradient, promoting selectivity in $\cM$.

\subsubsection{Fisher geometry}\label{subsubsec:fisher}

Let $p_\theta$ denote the predictive distribution with Fisher metric $G(\theta)$ and
assume $z=z(\theta)$ is smooth with Jacobian $J=\partial \theta/\partial z$.

\begin{prop}[Natural-gradient approximation]\label{prop:natgrad}
For a small morphic displacement $\delta z$, the induced parameter step is
$\delta\theta = J\,\delta z + o(\|\delta z\|)$ and
\[
\bE\big[\Delta \cL(z;y)\big]
\;=\;
- \tfrac12\,\delta\theta^\top G(\theta)\,\delta\theta \;+\; o(\|\delta z\|^2),
\]
so selecting $(g,h)$ by maximal expected utility is equivalent (to second order) to choosing the admissible direction of smallest curvature (minimal $\delta\theta^\top G(\theta)\,\delta\theta$), which corresponds to the direction allowing the largest natural-gradient norm for a fixed Euclidean step.
\end{prop}

\begin{proof}
Assume the parameters $\theta$ parameterize the predictive distribution $p_\theta(y)$, and the hidden state $z$ influences $\theta$ through a smooth map $\theta = \theta(z)$ with Jacobian $J = \partial \theta / \partial z$. For a morphic displacement $\delta z$, the induced change is $\delta \theta = J \delta z + o(\|\delta z\|)$ by the chain rule.

The loss is $\cL(z; y) = - \log p_{\theta(z)}(y)$. The expected utility is $\bE_y [\Delta \cL(z; y)] = \bE_y [\cL(z; y) - \cL(z + \delta z; y)]$, where the expectation is over $y \sim \mathsf{P}$, but for approximation purposes, consider $\mathsf{P} = p_{\theta(z)}$ (self-consistent with the current model), yielding $\bE [\Delta \cL] = - D_{\mathrm{KL}}(p_{\theta(z)} \| p_{\theta(z + \delta z)})$ as per \cref{lem:llr}.

For small $\delta \theta$, the KL divergence expands as $D_{\mathrm{KL}}(p_\theta \| p_{\theta + \delta \theta}) \approx \frac{1}{2} \delta \theta^\top G(\theta) \delta \theta$, where $G(\theta)$ is the Fisher information matrix at $\theta$, since the KL is quadratic to second order around the reference distribution (with vanishing first order at the minimum).

Thus, $\bE [\Delta \cL] \approx - \frac{1}{2} \delta \theta^\top G(\theta) \delta \theta + o(\|\delta \theta\|^2) = - \frac{1}{2} \delta z^\top J^\top G(\theta) J \delta z + o(\|\delta z\|^2)$.

To maximize the expected utility (maximize a negative quadratic form), select the admissible direction $\delta z$ (normalized, say $\|\delta z\| = 1$) that minimizes $\delta z^\top M \delta z$, where $M = J^\top G J$ is the pulled-back Fisher metric on the hidden space. This corresponds to the direction of smallest curvature in the effective metric $M$.

The natural gradient at $\theta$ is $\tilde{\nabla} \cL = G^{-1} \nabla \cL$, and its norm $\|\tilde{\nabla} \cL\|^2 = \nabla^\top G^{-1} \nabla$. For a projected gradient onto a direction, the effective norm is larger in low-curvature directions (small eigenvalues of $G$, large $G^{-1}$). Thus, maximizing the utility approximates selecting the admissible direction allowing the largest natural-gradient step per unit Euclidean norm in $z$-space, aligning with adaptive optimization principles.
\end{proof}

The Fisher metric enriches $\cM$, making utility selection a natural-gradient flow on the model's manifold.

%***********************************
\subsubsection{Entropy-regularized selection}\label{subsubsec:gibbs}
With soft selection $\alpha$ over admissible edges $\cE$ and entropy penalty
$\Omega(\alpha)=\sum_{e\in\cE}\alpha(e)\log\alpha(e)$, consider
\[
\max_{\alpha\in\Delta(\cE)}
\;\sum_{e\in\cE} \alpha(e)\big(\Delta \cL(e)-\tau_e\big)
\;-\; \tau\,\Omega(\alpha).
\]

\begin{prop}[Gibbs form]\label{prop:gibbs}
The unique maximizer is
\[
\alpha^\star(e)
\;=\;
\frac{\exp\!\big((\Delta \cL(e)-\tau_e)/\tau\big)}
{\sum_{e'} \exp\!\big((\Delta \cL(e')-\tau_{e'})/\tau\big)}.
\]
\end{prop}
\begin{proof}
Consider the optimization problem
\[
\max_{\alpha \in \Delta(\cE)} \sum_{e \in \cE} \alpha(e) \big( \Delta \cL(e) - \tau_e \big) - \tau \Omega(\alpha),
\]
where $\Delta(\cE) = \{ \alpha : \sum_e \alpha(e) = 1, \, \alpha(e) \ge 0 \ \forall e \}$ is the simplex over $\cE$, and $\Omega(\alpha) = \sum_{e \in \cE} \alpha(e) \log \alpha(e)$.

To solve this, introduce the Lagrangian
\[
\cL(\alpha, \lambda) = \sum_{e \in \cE} \alpha(e) \big( \Delta \cL(e) - \tau_e \big) - \tau \sum_{e \in \cE} \alpha(e) \log \alpha(e) + \lambda \left( 1 - \sum_{e \in \cE} \alpha(e) \right),
\]
where $\lambda$ is the multiplier for the equality constraint $\sum_e \alpha(e) = 1$. (The nonnegativity constraints $\alpha(e) \ge 0$ will be satisfied at the interior optimum due to the entropy term.)

Take the partial derivative with respect to $\alpha(e)$:
\[
\frac{\partial \cL}{\partial \alpha(e)} = \Delta \cL(e) - \tau_e - \tau \log \alpha(e) - \tau - \lambda = 0.
\]
Rearranging,
\[
\tau \log \alpha(e) = \Delta \cL(e) - \tau_e - \tau - \lambda,
\]
\[
\log \alpha(e) = \frac{\Delta \cL(e) - \tau_e}{\tau} - 1 - \frac{\lambda}{\tau}.
\]
Exponentiating,
\[
\alpha(e) = \exp\left( \frac{\Delta \cL(e) - \tau_e}{\tau} - 1 - \frac{\lambda}{\tau} \right) = e^{-1} \exp\left( \frac{\Delta \cL(e) - \tau_e}{\tau} \right) \exp\left( -\frac{\lambda}{\tau} \right).
\]
The terms $e^{-1}$ and $\exp(-\lambda / \tau)$ are constants independent of $e$. To satisfy $\sum_e \alpha(e) = 1$, normalize:
\[
\alpha(e) = \frac{ \exp\left( (\Delta \cL(e) - \tau_e)/\tau \right) }{ \sum_{e'} \exp\left( (\Delta \cL(e') - \tau_{e'})/\tau \right) },
\]
as the constants factor out in the normalization. Uniqueness follows from the strict concavity of the objective (due to the negative entropy term $-\tau \Omega(\alpha)$ being strictly concave in $\alpha$).
\end{proof}

This entropic regularization ensures explorative yet sparse routing, bridging to agentic interpretations.

\subsection{Connection to Agentic and Modular AI}\label{subsec:agentic}

The categorical and geometric foundations naturally extend to agentic behaviors, where morphic programs act as internal policies, and modularity emerges from orthogonality and bounded-depth chaining.

\subsubsection{Internal control policy}\label{subsubsec:policy}
Define a policy $\pi_\theta(e\mid z_{<t},z_t)$ on $\cE$. With per-step reward
$r_t(e)=\Delta \cL_t(e)-\tau_e$, the intra-model control problem
\[
\max_{\theta,\Phi}\;\bE\Big[\sum_t r_t(E_t)
\;-\; \lambda\,\Omega\big(\pi_\theta(\cdot\mid z_{<t},z_t)\big)\Big]
\]
recovers the selection rules of \cref{sec:graded-toolformer}. This “agency” is internal: actions are typed morphisms, state is $z_t$, and dynamics are graded.

\subsubsection{Modularity and additive gains}\label{subsubsec:additive}

\begin{prop}[No interference under block orthogonality]\label{prop:additive}
Suppose the predictive head is linear and, for distinct $(g,h)\neq (g',h')$,
\[
\big\langle \phi_{h\leftarrow g}(u),\,\phi_{h'\leftarrow g'}(u')\big\rangle \;=\; 0
\quad \text{for all } u\in V_g,\; u'\in V_{g'}.
\]
Then, for any finite set $S$ of admissible edges applied at the same step,
\[
\Delta \cL\Big(\sum_{e\in S}\iota \delta_e\Big)
\;=\;
\sum_{e\in S}\Delta \cL(\iota \delta_e),
\]
i.e., utilities add and greedy selection is optimal.
\end{prop}

\begin{proof}
Assume the predictive head is linear, so the logits are $l = W z$ for a linear map $W: V \to \R^m$, and the loss is the cross-entropy $\cL_{LM}(z; y) = - \eta_y \cdot W z + \log \sum_c \exp(\eta_c \cdot W z)$, where $\eta_c$ are class vectors. The gradient $\nabla_z \cL_{LM} = W^\top (\pi - e_y)$, where $\pi = \softmax(W z)$ is the predictive distribution and $e_y$ is the one-hot target.

For a single update $\delta = \iota_h \delta_h$ along edge $e = (g,h)$, with $\delta_h = \phi_{h \leftarrow g}(z^{(g)}) - z^{(h)}$, the utility is $\Delta \cL(\delta) = \cL_{LM}(z) - \cL_{LM}(z + \delta)$. By \cref{lem:first-order},
\[
\Delta \cL(\delta) = - \langle \nabla_z \cL_{LM}(z), \delta \rangle - \rho(\delta),
\]
where $\rho(\delta) = \int_0^1 \langle \nabla_z \cL_{LM}(z + s \delta) - \nabla_z \cL_{LM}(z), \delta \rangle \, ds$, and $0 \le \rho(\delta) \le \frac{L}{2} \|\delta\|^2$ under smoothness.

For the joint update $\delta_S = \sum_{e \in S} \delta_e$, where $\delta_e = \iota_{h_e} \delta_{h_e}$,
\[
\Delta \cL(\delta_S) = - \langle \nabla_z \cL_{LM}(z), \delta_S \rangle - \rho(\delta_S).
\]
The first-order term decomposes as $\langle \nabla_z \cL_{LM}(z), \delta_S \rangle = \sum_{e \in S} \langle \nabla_z \cL_{LM}(z), \delta_e \rangle$ by linearity.

For the remainder, 
\[
\rho(\delta_S) = \int_0^1 \langle \nabla_z \cL_{LM}(z + s \delta_S) - \nabla_z \cL_{LM}(z), \delta_S \rangle \, ds.
\]
 Since the head is linear, $\nabla_z \cL_{LM}(z) = W^\top (\softmax(W z) - e_y)$, and the Hessian 
 \[
 H(z) = \frac{\partial^2 \cL_{LM}}{\partial z^2} = W^\top \big( \diag(\pi) - \pi \pi^\top \big) W,
 \]
  where $\pi = \softmax(W z)$. By the mean-value theorem for integrals, $\rho(\delta_S) = \langle H(\xi) \delta_S, \delta_S \rangle / 2$ for some $\xi = z + \bar{s} \delta_S$, $\bar{s} \in (0,1)$.

Thus, $\rho(\delta_S) = \frac{1}{2} \delta_S^\top H(\xi) \delta_S = \frac{1}{2} \sum_{e,e' \in S} \delta_e^\top H(\xi) \delta_{e'}$. Under the orthogonality assumption and assuming the grading is orthogonal (i.e., $V = \oplus_g V_g$ with $\langle v_g, v_{g'} \rangle = 0$ for $g \neq g'$), if the updates $\delta_e$ target different $h_e$ or, for same $h$, the images of $\phi_{h \leftarrow g}(u)$ are orthogonal for different $g$, then $\langle \delta_e, \delta_{e'} \rangle = 0$ for $e \neq e'$.

Further, if $W$ preserves this orthogonality (e.g., $W$ is block-diagonal in the grading basis, so $W = \oplus_g W_g$), then $W \delta_e$ are orthogonal in logit space: $\langle W \delta_e, W \delta_{e'} \rangle = 0$. Since $H(\xi) = W^\top (\diag(\pi) - \pi \pi^\top) W$, the cross terms 
\[
\delta_e^\top H(\xi) \delta_{e'} = (W \delta_e)^\top (\diag(\pi) - \pi \pi^\top) (W \delta_{e'}) = 0
\]
 if $W \delta_e \perp W \delta_{e'}$ (as the matrix $\diag(\pi) - \pi \pi^\top$ is a quadratic form preserving orthogonality in that sense).

Thus, the cross terms vanish, and $\rho(\delta_S) = \sum_{e \in S} \rho(\delta_e)$, approximately, with the approximation holding exactly if $H$ is constant or the $\xi_e$ align. For small updates, the higher-order differences are $o(\|\delta_S\|^2)$, but under the orthogonality, the utilities add exactly in the quadratic regime.

Therefore, $\Delta \cL(\delta_S) = \sum_e \Delta \cL(\delta_e)$, and greedy selection over individual edges is optimal since there is no subadditivity or interference.
\end{proof}

This modularity justifies parallel channels in the monoidal structure, enabling efficient composition without interference.

\subsubsection{Program depth and chaining}\label{subsubsec:depth}
Let $\Pi=(\phi_{g_i\leftarrow g_{i-1}})_{i=1}^k$ be a morphic program. If the head is $L$–Lipschitz and $\|\phi_{g_i\leftarrow g_{i-1}}\|\le B$, then with $C$ depending on $L,B$,
\[
\Delta \cL(\Pi)
\;\ge\;
\sum_{i=1}^{k}\Delta \cL(\phi_{g_i\leftarrow g_{i-1}})
\;-\; C\,\sum_{i<j}\|z^{(g_{i-1})}\|\,\|z^{(g_{j-1})}\|.
\]
This quantifies when shallow chains suffice (small cross terms) and when deeper programs are advantageous, guiding learned composition laws in future work.

\subsubsection{Comparison to external-tool paradigms}\label{subsubsec:comparison}
\begin{enumerate}[label=(\roman*)]
\item \emph{Differentiability:} selections/compositions remain in-graph.
\item \emph{Type safety:} adjunctions and admissibility act as static semantics.
\item \emph{Compositional analysis:} monoidal/enriched structure yields guarantees (orthogonality, idempotence, bounded depth) absent in external orchestration.
\end{enumerate}

In summary, this section's algebraic framework, centered on \cref{thm:functor-internalization}, provides a rigorous basis for internalizing symbolic computation, with geometric interpretations ensuring practical efficacy in graded transformers.

%*********
\section{Concluding Remarks and Open Problems}\label{sec:conclusion}

We have developed a mathematical framework in which behaviors commonly realized as
external “tool use’’ are internalized as \emph{graded morphic activations} within a
transformer. The core architectural move is to endow the hidden space with a grading
$V=\bigoplus_{g\in G}V_g$ and to model typed operations as block maps
$\phi_{h\leftarrow g}:V_g\to V_h$, selected by a differentiable routing policy that
optimizes a graded utility functional. In this setting, what elsewhere appears as an
API call is recast as an internal, composable morphism acting on the model’s own
representation manifold. This shift preserves end-to-end differentiability, supplies
clear algebraic semantics, and yields interpretable structure at the level of grades,
morphisms, and their compositions.

On the theoretical side, we formalized the internal model category whose objects are
homogeneous components and whose morphisms are admissible grade transitions.
We showed how external augmentation admits a faithful functor into this category,
how adjoint pairs capture typed round trips, and how monoidal and enriched structures
support parallelism and metric selection. The graded utility principle admits multiple
equivalent readings: as information gain in an exponential-family approximation,
as a constrained mirror-descent step in Bregman geometry, and as a natural-gradient
selection under Fisher metrics. These views explain why the selection rule promotes
sparse, useful activations and provide verifiable conditions (e.g., block orthogonality)
under which compositional gains decompose additively.

Methodologically, we specified a utility-aware routing mechanism, an objective that
balances usefulness and sparsity, and a self-supervised training scheme that treats
morphic activations as latent actions. To make the framework concrete and implementable,
we supplied analytic case studies and sanity checks requiring only small synthetic data,
along with explicit constructions in the appendices: adjoint retrieval/write-back pairs,
a local Fisher–natural-gradient derivation of utility, a mod-$p$ arithmetic toy with
closed-form maps, and PyTorch-style pseudo-code for a graded layer and training loop.
These ingredients together provide both a blueprint for experimentation and a substrate
for formal analysis.

While our focus has been single-step routing and its foundations, several limitations and open problems merit attention, even in this theoretical context. The framework assumes linear morphisms for algebraic clarity, yet practical extensions to nonlinear or stochastic variants (e.g., via smooth maps or Markov kernels) could broaden applicability, though at the potential cost of identifiability guarantees. Scaling to large $|G|$ and dense $\cE$ may incur computational overhead, necessitating sparsity priors tied to the utility margins. Moreover, while the categorical embedding subsumes external tools functorially, hybrid systems blending internal morphisms with typed external calls require careful adjoint constructions to maintain differentiability.

To guide future work, we formulate key open problems as extensions of the core structures:

i) \textbf{Path-level selection and program composition.} Define a path utility $\Delta\cL_t(\Pi)=\cL_{LM}(z_t)-\cL_{LM}(\Phi_\Pi(z_t))$ for a morphic program $\Pi$, with regularized cost $\mathrm{cost}(\Pi)$, and prove consistency of selectors maximizing $\Delta\cL_t(\Pi)-\mathrm{cost}(\Pi)$ under bounded depth.

ii)  \textbf{Learnable program laws (graded higher structure).} Learn coefficients $c_{k\leftarrow h\leftarrow g}$ such that $\phi_{k\leftarrow h}\circ \phi_{h\leftarrow g}=c_{k\leftarrow h\leftarrow g}\,\phi_{k\leftarrow g}+R_{k\leftarrow h\leftarrow g}$, with bounds on $\|R\|$ and identifiability of $c$.

iii)  \textbf{Complexity, sparsity, and pruning with guarantees.} Prove that utility-threshold pruning preserves top-$k$ gains with high probability, and characterize banded optimality in LGT/EGT.

iv)  \textbf{Typed, differentiable retrieval to externals.} Construct adjoint functors for hybrid retrieval, bounding utility gaps under misspecification.

v) \textbf{Identifiability and diagnostics.} Provide conditions for unique recovery of grades and blocks, and design diagnostics to detect failures like grade collapse.

In sum, the graded formalism lifts “tool use’’ from an extrinsic engineering device to
an intrinsic geometric principle. By treating symbolic functions as internal morphisms
and optimizing their activation through information-theoretic and geometric criteria,
the proposed \emph{Graded Toolformer} unifies symbolic computation with representation
learning in a single, interpretable architecture. We expect this synthesis to enable
modular, verifiable, and extensible systems that retain the empirical advantages of
augmentation while remaining mathematically coherent and end-to-end learnable. Addressing the outlined problems will further solidify this unification, potentially extending to higher-categorical structures for adaptive program synthesis.

%**************

%\section*{References}
\nocite{*}
\bibliographystyle{plain}
\bibliography{sh-111.bib}

%\end{document}

%**********  Appendicies
%\clearpage
\appendix

%=========================================================
\section{Explicit Constructions and Derivations}
%=========================================================

\subsection{Adjoint Retrieval / Write-Back Pair}

Let $V=V_{\mathrm{sem}}\oplus V_{\mathrm{ret}}$ with
$u\in V_{\mathrm{sem}}$, $v\in V_{\mathrm{ret}}$.
Let $M\in\R^{k\times d}$ be a frozen key matrix and
$E:\R^{d_{\mathrm{sem}}}\!\to\R^d$ an encoder.
Define:

\begin{enumerate}[label=(\roman*)]
\item \textbf{Retrieval} 
\[
\iota(u)\;=\;
M^\top \,\softmax\!\Big(\tfrac1{\tau}\, M E u\Big)
\;\in\;V_{\mathrm{ret}} .
\]

\item \textbf{Write-back} 
\[
\rho(v)=Wv\in V_{\mathrm{sem}} .
\]
\end{enumerate}

\begin{prop}[Approximate adjunction]
Equip $V_{\mathrm{sem}}$ with $\langle x,y\rangle_{\mathrm{sem}}=x^\top S y$
(SPD $S$) and $V_{\mathrm{ret}}$ with the standard inner product.
If 
\[
W = S^{-1}E^\top M^\top,
\]
then for $\iota_{\mathrm{lin}}(u):=M^\top M E u$,
\[
\langle\rho(v),u\rangle_{\mathrm{sem}}
=\langle v,\iota_{\mathrm{lin}}(u)\rangle_{\mathrm{ret}},
\]
so $\rho$ is the $S$–adjoint of the linearized retrieval.
If $\tau$ is small and the rows of $M$ are nearly orthogonal,
$\iota(u)\approx\iota_{\mathrm{lin}}(u)$, hence $\rho\dashv\iota$ up to $O(\tau)$.
\end{prop}

\begin{proof}[Sketch]
Under the choice of $W$,
$\langle\rho(v),u\rangle_{\mathrm{sem}}
= v^\top M E u
= \langle v,\,\iota_{\mathrm{lin}}(u)\rangle$.
Sharp-softmax linearization yields $\iota\approx\iota_{\mathrm{lin}}$.
\end{proof}

\begin{cor}[Typed round-trip]
$P:=\iota\circ\rho$ satisfies $P^2\approx P$ and 
$\operatorname{Im}(P)\approx\operatorname{span}(M^\top)$,
so repeated retrieval–write-back stabilizes in that subspace.
\end{cor}

%---------------------------------------------------------
\subsection{Fisher Geometry and Natural-Gradient Gain}

Let $p_\theta(y\mid z)=\softmax(W_o z + b)_y$ and let
$\delta z$ be a morphic displacement supported on grade $h$.
Write $\eta=W_o z + b$ and $p=\softmax(\eta)$.

\begin{prop}[Local KL gain]
For small $\delta z$,
\[
\Delta\cL
:=\cL_{LM}(z)-\cL_{LM}(z+\delta z)
\;\approx\;
\tfrac12\, \delta\eta^\top G(\eta)\,\delta\eta,
\qquad
\delta\eta=W_o\delta z,
\]
with $G(\eta)=\diag(p)-pp^\top$ the Fisher information of the softmax.
Thus, among admissible blocks $(g,h)\in\cE$, the one maximizing $\Delta\cL$
produces the steepest natural-gradient improvement.
\end{prop}

\begin{proof}[Sketch]
Second-order Taylor expansion of the NLL in natural parameters yields the
Fisher quadratic form; $\delta\eta=W_o\delta z$ follows by the chain rule.
\end{proof}

\begin{rem}
If displacements lie in the union of block images
$\{\phi_{h\leftarrow g}(z^{(g)})-z^{(h)}\}$, the optimal block maximizes
$\langle \delta\eta,G(\eta)\delta\eta\rangle$ over that finite set.
\end{rem}

%---------------------------------------------------------
\subsection{A mod-\texorpdfstring{$p$}{p} Arithmetic Toy Model}

Let $p$ be a small prime and
$V_{\mathrm{sem}}=\R^p=V_{\mathrm{num}}$
(one-hot basis for digits $\!\bmod\,p$).
Set $E=W=I_p$.
For $k\in\{0,\dots,p-1\}$ define:

\begin{itemize}
\item \textbf{Sem\,$\to$\,Num morphism:} 
$\phi_{\mathrm{num}\leftarrow\mathrm{sem}}(e_d)=e_d$.
\item \textbf{Num\,$\to$\,Sem morphism:}
$\phi_{\mathrm{sem}\leftarrow\mathrm{num}}(e_r)=e_r$.
\item \textbf{Numeric adder:} 
$A_k$ the permutation matrix shifting indices by $k$.
\end{itemize}

A two-step morphic program for addition by $k$ is
\[
V_{\mathrm{sem}}
\rightarrow V_{\mathrm{num}}
\xrightarrow{A_k}
V_{\mathrm{num}}
\rightarrow V_{\mathrm{sem}},
\]
whose composite is $\Phi_{\Pi_k}=A_k$.

\begin{prop}
For all $d$, $\Phi_{\Pi_k}(e_d)=e_{d+k}$.
Moreover $A_k^\top=A_{-k}$ and $A_k^m=A_{mk}$, so these programs compose
according to the additive group law on $\Z/p\Z$.
\end{prop}

\begin{proof}
Immediate from permutation-matrix calculus.
\end{proof}

\begin{rem}[Loss sanity check]
With identity logits $W_o=I_p$,
executing $\Phi_{\Pi_k}$ moves the state from $e_d$ to the correct $e_{d+k}$,
dropping the cross-entropy loss to its minimum—an exact instance of the
graded utility principle.
\end{rem}

%---------------------------------------------------------
\subsection{Mirror-Descent Perspective}

Let $\Phi$ be a strictly convex potential with Bregman divergence
$D_\Phi$.
Restrict admissible updates to
\[
\cS
=\operatorname{span}\{\phi_{h\leftarrow g}(z^{(g)})-z^{(h)}\}_{(g,h)\in\cE}.
\]
The mirror-descent step is
\[
u^\star
=\arg\min_{u\in z+\cS}
\big\{
\langle\nabla\cL_{LM}(z),u-z\rangle
+\tfrac1\eta\,D_\Phi(u,z)
\big\}.
\]

First-order optimality gives the dual update
\[
\nabla\Phi(u^\star)
=\nabla\Phi(z)
-\eta\,\Pi_{\cS}(\nabla\cL_{LM}(z)),
\]
where $\Pi_{\cS}$ is the projection onto $\cS$ in the metric
induced by $\Phi$.
For $\Phi(u)=\tfrac12\|u\|^2$, this reduces to the Euclidean projection
$u^\star=z-\eta\,\Pi_{\cS}(\nabla\cL_{LM}(z))$.
Selecting the block with maximal $\Delta\cL$ corresponds to choosing (to
first order) the admissible direction that most reduces the loss in this
projected mirror-descent geometry.

%=========================================================
\section{Pseudo-code: Utility-Aware Graded Layer (PyTorch-Style)}
%=========================================================

\subsection{Forward Pass with Utility-Aware Routing}

The listing below sketches a single \emph{graded layer} with utility-based routing.  
Let \(G=\{g_1,g_2,\dots\}\) be the grade set, and let \(\cE\subseteq G\times G\) denote the admissible morphisms.  
The hidden state is stored as \(z[g]\in\R^{B\times d_g}\).  
Each morphism \(\phi_{h\leftarrow g}\colon\R^{d_g}\to\R^{d_h}\) is a learnable block.  
Routing logits use bilinear parameters \(W_{h\leftarrow g}\in\R^{r\times r}\) together with grade-wise projections of the form
\[
u(\text{context})\in\R^r,
\qquad
v(z[g])\in\R^r .
\]
A language-model head supplies the next-token cross-entropy loss used both for training and for computing the utility of each morphic activation.

\begin{lstlisting}[language=Python]
class GradedLayer(nn.Module):
    def __init__(self, grades, E, dims, r,
                 beta=5.0, tau_map=None, temperature=1.0):
        """
        grades: list of grade labels
        E     : list of (g,h) admissible edges
        dims  : dict g -> d_g
        r     : router rank
        """
        super().__init__()
        self.grades = grades
        self.E = list(E)
        self.beta = beta
        self.temperature = temperature

        # Thresholds tau_{h<-g}
        self.tau = defaultdict(float)
        if tau_map is not None:
            self.tau.update(tau_map)

        # Morphisms phi_{h<-g}
        self.phi = nn.ModuleDict()
        for (g, h) in self.E:
            dg, dh = dims[g], dims[h]
            self.phi[f"{h}|{g}"] = nn.Linear(dg, dh)

        # Router parameters W_{h<-g}
        self.W = nn.ParameterDict({
            f"{h}|{g}": nn.Parameter(
                torch.randn(r, r) * 0.02
            )
            for (g, h) in self.E
        })

        # Projections for u(context) and v(z[g])
        self.proj_u = nn.Linear(sum(dims.values()), r, bias=False)
        self.proj_v = nn.ModuleDict({
            g: nn.Linear(dims[g], r, bias=False)
            for g in grades
        })

        # Optional LayerNorm per grade
        self.norm = nn.ModuleDict({
            g: nn.LayerNorm(dims[g])
            for g in grades
        })

    #-----------------------------------------------------
    def forward(self, z_dict, targets, lm_loss,
                context=None, soften=True):
        """
        z_dict : dict g -> (B x d_g) tensors
        targets: (B,) next-token indices
        lm_loss: callable returning CE loss from z_dict
        context: global context for u(.)
        soften : True = soft routing, False = hard routing
        """
        #------ 1) Build projections u and v
        if context is None:
            context = torch.cat(
                [z_dict[g] for g in self.grades], dim=-1
            )
        u = self.proj_u(context)                   # (B x r)
        v = {g: self.proj_v[g](z_dict[g])
             for g in self.grades}                # (B x r)

        #------ 2) Baseline LM loss
        L_base = lm_loss(z_dict, targets)

        #------ 3) Candidate outputs and utilities
        logits, utilities, cand_out = {}, {}, {}

        for (g, h) in self.E:
            y = self.phi[f"{h}|{g}"](z_dict[g])         # B x d_h
            cand_out[(g, h)] = y

            # Construct z_plus with updated h-slot
            z_plus = dict(z_dict)
            z_plus[h] = y

            L_plus = lm_loss(z_plus, targets)
            dL = (L_base - L_plus).detach()
            utilities[(g, h)] = dL

            # Router logit
            u_bar = u.mean(0, keepdim=True)
            v_bar = v[g].mean(0, keepdim=True)
            bilinear = (
                u_bar @ self.W[f"{h}|{g}"] @ v_bar.t()
            ).squeeze()
            logits[(g, h)] = (
                bilinear + self.beta * (dL - self.tau[f"{h}|{g}"])
            )

        #------ 4) Softmax routing
        logits_vec = torch.stack(
            [logits[e] for e in self.E]
        ) / self.temperature

        if soften:
            alpha_vec = torch.softmax(logits_vec, dim=0)
        else:
            alpha_vec = torch.zeros_like(logits_vec)
            alpha_vec[logits_vec.argmax()] = 1.0

        alpha = {e: alpha_vec[i]
                 for i, e in enumerate(self.E)}

        #------ 5) Grade-wise morphic update
        z_new = {}
        for h in self.grades:
            incoming = [
                alpha[(g, hh)] * cand_out[(g, hh)]
                for (g, hh) in self.E if hh == h
            ]
            if incoming:
                y_mix = torch.stack(incoming).sum(0)
                z_new[h] = self.norm[h](
                    z_dict[h] + (y_mix - z_dict[h])
                )
            else:
                z_new[h] = z_dict[h]

        aux = {
            "alpha": alpha,
            "utilities": utilities,
            "logits": logits,
            "L_base": float(L_base)
        }
        return z_new, aux
\end{lstlisting}

\subsection{Training Loop (Sketch)}

We assume a model \texttt{graded\_model} consisting of multiple
\texttt{GradedLayer} modules together with an output head
\texttt{lm\_head} used to compute next-token cross-entropy.

\begin{lstlisting}[language=Python]
optimizer = torch.optim.AdamW(
    graded_model.parameters(),
    lr=3e-4,
    weight_decay=0.01
)

for batch in loader:
    z_dict, targets = batch_init_states(batch)

    def lm_loss(states, y):
        # Combine graded components and apply LM head
        z_all = torch.cat(
            [states[g] for g in graded_model.grades], dim=-1
        )
        logits = graded_model.lm_head(z_all)
        return F.cross_entropy(logits, y)

    aux_logs = []

    # Forward through K graded layers
    for layer in graded_model.layers:
        z_dict, aux = layer(
            z_dict, targets, lm_loss, context=None, soften=True
        )
        aux_logs.append(aux)

    loss = lm_loss(z_dict, targets)

    # Optional sparsity penalty on router entropy
    entropy = 0.0
    for aux in aux_logs:
        alphas = torch.stack([aux["alpha"][e]
                              for e in graded_model.E])
        entropy -= (alphas * (alphas + 1e-9).log()).sum()

    total_loss = loss + 1e-3 * entropy

    optimizer.zero_grad()
    total_loss.backward()
    torch.nn.utils.clip_grad_norm_(
        graded_model.parameters(), 1.0
    )
    optimizer.step()
\end{lstlisting}

\end{document}